\documentclass{article}

    \PassOptionsToPackage{numbers, compress}{natbib}

\usepackage[final]{neurips_2021}

\usepackage[utf8]{inputenc} 
\usepackage[T1]{fontenc}    
\usepackage[colorlinks,
            linkcolor=red,
            anchorcolor=blue,
            citecolor=blue
            ]{hyperref}
\usepackage{url}            
\usepackage{booktabs}       
\usepackage{amsfonts}       
\usepackage{nicefrac}       
\usepackage{microtype}      
\usepackage{xcolor}

\usepackage{smile}  
\allowdisplaybreaks
\newcommand{\la}{\langle}
\newcommand{\ra}{\rangle}
\newcommand{\qvalue}{Q}
\newcommand{\vvalue}{V}

\newcommand{\reward}{r}

\def \CCC {}
\def \algname {\text{FLUTE}}
\def \alg {\mathtt{Alg}}

\usepackage{enumitem}

\usepackage[textsize=tiny]{todonotes}
\title{Uniform-PAC Bounds for Reinforcement Learning with Linear Function Approximation}

\author{%
  Jiafan He \\
  Department of Computer Science\\
 University of California, Los Angeles\\
  CA 90095, USA \\
  \texttt{jiafanhe19@ucla.edu} \\
   \And
    Dongruo Zhou \\
  Department of Computer Science\\
 University of California, Los Angeles\\
  CA 90095, USA \\
  \texttt{drzhou@cs.ucla.edu} \\
     \And
    	Quanquan Gu \\
  Department of Computer Science\\
 University of California, Los Angeles\\
  CA 90095, USA \\
  \texttt{ qgu@cs.ucla.edu} \\
}

\date{}

\usepackage{smile}

\begin{document}
\maketitle

\begin{abstract}
We study reinforcement learning (RL) with linear function approximation. Existing algorithms for this problem only have high-probability regret and/or Probably Approximately Correct (PAC) sample complexity guarantees, which cannot guarantee the convergence to the optimal policy. In this paper, in order to overcome the limitation of existing algorithms, we propose a new algorithm called FLUTE, which enjoys uniform-PAC convergence to the optimal policy with high probability. The uniform-PAC guarantee is the strongest possible guarantee for reinforcement learning in the literature, which can directly imply both PAC and high probability regret bounds, making our algorithm superior to all existing algorithms with linear function approximation.
At the core of our algorithm is a novel minimax value function estimator and a multi-level partition scheme to select the training samples from historical observations. Both of these techniques are new and of independent interest. 
\end{abstract}







\section{Introduction}



Designing efficient reinforcement learning (RL) algorithms for environments with large state and action spaces is one of the main tasks in the RL community. To achieve this goal, function approximation, which uses a class of predefined functions to approximate either the value function or transition dynamic, has been widely studied in recent years. Specifically, a series of recent works \citep{jiang2017contextual,jin2019provably,modi2020sample,zanette2020learning,ayoub2020model,zhou2020nearly} have studied RL with linear function approximation with provable guarantees. They show that with linear function approximation, one can either obtain a sublinear regret bound against the optimal value function \citep{jin2019provably,zanette2020learning,ayoub2020model,zhou2020nearly} or a polynomial sample complexity bound \citep{kakade2003sample} (Probably Approximately Correct (PAC) bound for short) in finding a near-optimal policy {\citep{jiang2017contextual,modi2020sample}}.  

However, neither the regret bound or PAC bound is a perfect performance measure. As discussed in detail by \cite{dann2017unifying}, these two measures fail to guarantee the convergence to the optimal policy. Therefore, an algorithm with high probability regret and/or PAC bound guarantees do not necessarily learn the optimal policy, and can perform badly in practice. 
In detail, one can face the following two situations:
\begin{itemize}[leftmargin = *]
    \item An algorithm with a sublinear regret suggests that the summation of the suboptimality gaps $\Delta_t$ (the gap between the values of the current adapted policy and optimal policy, see Definition \ref{def:Regret} for details.) in the first $T$ rounds is bounded by $o(T)$. However, this algorithm may be arbitrarily suboptimal infinitely times \footnote{Suppose the suboptimality gaps satisfy $\Delta_t = \ind\{t = i^2, i = 1,\dots\}$, then the regret in the first $T$ rounds is upper bounded by $O(\sqrt{T})$, and the constant $1$-gap will appear infinitely often }, thus it fails to converge to the optimal policy.  
    \item An algorithm is $(\epsilon,\delta)$-PAC suggests that with probability at least $1-\delta$, the number of suboptimality gaps $\Delta_t$ that are greater than $\epsilon$ will be at most polynomial in $\epsilon$ and $\log(1/\delta)$. The formal definition of $(\epsilon,\delta)$-PAC can be found in Definition \ref{def:pac}. However, this algorithm may still have gaps satisfying $\epsilon/2 < \Delta_t < \epsilon$ infinitely often, thus fails to converge to the optimal policy. 
\end{itemize}

To overcome the limitations of regret and PAC guarantees,  \citet{dann2017unifying} proposed a new performance measure called uniform-PAC, which is a strengthened notion of the PAC framework. Specifically, an algorithm is uniform-PAC if there exists a function of the target accuracy $\epsilon$ and the confidence parameter $\delta$ that upper bounds the number of suboptimality gaps satisfying $\Delta_t>\epsilon$ \emph{simultaneously} for all $\epsilon>0$ with probability at least $1-\delta$. 
The formal definition of uniform-PAC can be found in Definition \ref{def:uni-pac}. Algorithms that are uniform-PAC converge to an optimal policy with high probability, and yield both PAC and high probability regret bounds. In addition, they proposed a UBEV algorithm for learning tabular MDPs, which is uniform-PAC. Nevertheless, UBEV is designed for tabular MDPs, and it is not clear how to incorporate function approximation into UBEV to scale it up for large (or even infinite) state and action space. Therefore, a natural question arises:

\begin{center}
   \emph{Can we design a provable efficient uniform-PAC RL algorithm with linear function approximation?} 
\end{center}
In this work, we answer this question affirmatively. In detail, we propose new algorithms for both contextual linear bandits and linear Markov decision processes (MDPs) \citep{yang2019sample,jin2019provably}. Both of them are uniform-PAC, and their sample complexity is comparable to that of the state-of-the-art algorithms which are not uniform-PAC. Our key contributions are highlighted as follows.
\begin{itemize}[leftmargin = *]
    \item We begin with contextual linear bandits problem as a ``warm-up'' example of the RL with linear function approximation (with horizon length equals $1$). We propose a new algorithm called uniform-PAC OFUL (UPAC-OFUL), and show that our algorithm is uniform-PAC with $\tilde O(d^2/\epsilon^2)$ sample complexity, where $d$ is the dimension of contexts and $\epsilon$ is the accuracy parameter. In addition, this result also implies an $\tilde O(d\sqrt{T})$ regret in the first $T$ round and matches the result of OFUL algorithm \citep{abbasi2011improved} up to a logarithmic factor. 
    The key idea of our algorithm is a novel minimax linear predictor and a multi-level partition scheme to select the training samples from past observations. To the best of our knowledge, this is the first algorithm with a uniform PAC-bound for contextual bandits problems. 
    \item We also consider RL with linear function approximation in episodic linear MDPs, where the transition kernel admits a low-rank factorization. We propose an algorithm dubbed uniForm-PAC Least-sqUare value iTEration ($\algname$), which adapts the novel techniques we developed in the contextual linear bandits setting, and show that our algorithm is uniform-PAC with $\tilde O(d^3H^5/\epsilon^2)$ sample complexity, where $d$ is the dimension of the feature mapping, $H$ is the length of episodes and $\epsilon$ is the accuracy parameter. This result further implies an $\tilde O(\sqrt{d^3H^4T})$ regret in the first $T$ steps and matches the result of LSVI-UCB algorithm \citep{jin2019provably} up to a $\sqrt{H}$-factor, while LSVI-UCB is not uniform-PAC.
    Again, $\algname$ is the first uniform-PAC RL algorithm with linear function approximation. 
\end{itemize}

\noindent\textbf{Notation} 
We use lower case letters to denote scalars, and use lower and upper case bold face letters to denote vectors and matrices respectively. For any positive integer $n$, we denote by $[n]$ the set $\{1,\dots,n\}$. For a vector $\xb\in \RR^d$ , we denote by $\|\xb\|_1$ the Manhattan norm and denote by $\|\xb\|_2$ the Euclidean norm. For a vector $\xb\in \RR^d$ and matrix $\bSigma\in \RR^{d\times d}$, we define $\|\xb\|_{\bSigma}=\sqrt{\xb^\top\bSigma\xb}$. For two sequences $\{a_n\}$ and $\{b_n\}$, we write $a_n=O(b_n)$ if there exists an absolute constant $C$ such that $a_n\leq Cb_n$. We use $\tilde O(\cdot)$ to further hide the logarithmic factors. For logarithmic regret, we use $\tilde O(\cdot)$ to hide all logarithmic terms except $\log T$.

\section{Related Work}

\subsection{Linear bandits}
There is a series of works focusing on the stochastic linear bandits problem. These works can be categorized into two groups: the works aim at providing sublinear regret guarantee
and the works providing PAC bound for linear best-arm identification problem.
In detail, for finite action set with $K$ arms, \citet{auer2002using} proposed a SupLinRel algorithm which achieves an $O(\sqrt{dT \log^3(TK)})$ regret, where $K$ is the number of arms. \citet{chu2011contextual} proposed a SupLinUCB algorithm which has the same regret bound as SupLinRel, but is easier to implement. 
\citet{li2019nearly} proposed a Variable-Confidence-Level (VCL) SupLinUCB algorithm and improved the regret bound to $O(\sqrt{dT \log T \log K})$. For infinite action set, \citet{dani2008stochastic} proposed a Confidence Ball algorithm with an $O(d\sqrt{T\log^3 T})$ regret and proved an $\Omega(d\sqrt{T})$ lower bound.  \citet{abbasi2011improved} proposed OFUL algorithm and improved the regret bound to $O(d\sqrt{T\log^2 T})$. When the reward has a bounded variance,
\citet{zhou2020nearly,zhang2021variance} proposed algorithms with variance-aware confident sets, and obtained tight variance-dependent regret bounds. For the best-arm identification problem, to find an arm which is $\epsilon$-suboptimal, \citet{soare2014best} proposed a G-allocation strategy with an $\tilde O(d/\epsilon^2)$ sample complexity. \citet{karnin2016verification} proposed an Explore-Verify framework which improves the sample complexity by some logarithmic factors. \citet{xu2018fully} proposed a LinGapE algorithm whose sample complexity matches the lower bound up to some $K$ factors. \citet{tao2018best} proposed an ALBA algorithm which improves the sample complexity to have a linear dimension dependence.  \citet{fiez2019sequential} studied transductive linear bandits and proposed an algorithm with sample complexity similar to linear bandits. Compared with best-arm identification, the contextual bandits setting we focus on is more challenging since the action set will change at each round. 

\subsection{RL with linear function approximation}
Recently, a line of work focuses on analyzing RL with linear function approximation. To mention a few, 
\citet{jiang2017contextual} studied MDPs with low Bellman rank and proposed an OLIVE algorithm, which has the PAC guarantee.
\citet{yang2019sample} studied the linear transition model and proposed a sample-optimal Q-learning method with a generative model. \citet{jin2019provably} studied the linear MDP model and proposed an LSVI-UCB algorithm under the online RL setting (without a generative model) with $\tilde O(\sqrt{d^3H^3T})$ regret. Later, \citet{zanette2019tighter} studied the low inherent Bellman error model and proposed an ELEANOR algorithm with a better regret $\tilde O(dH\sqrt{T})$ using a global planning oracle. 
\citet{modi2020sample} studied the linearly combined model ensemble and proposed a provable sample-efficient algorithm. 
\citet{jia2020model, ayoub2020model} studied the linear mixture MDPs and proposed a UCRL-VTR algorithm with an $\tilde O(d\sqrt{H^3T})$ regret. Recently \citet{zhou2020nearly} improved the regret bound to $\tilde O(dH\sqrt{T})$ with a new algorithm design and a new Bernstein inequality. However, all of these works aim at deriving PAC sample complexity guarantee or regret bound, and none of them has the uniform PAC guarantee for learning MDPs with linear function approximation. Our work will fill this gap in the linear MDP setting \citep{yang2019sample,jin2019provably}.

\section{Preliminaries}
We consider episodic Markov Decision Processes (MDPs) in this work. Each episodic MDP is denoted by a tuple $M(\cS, \cA, H, \{\reward_h\}_{h=1}^H, \{\PP_h\}_{h=1}^H)$. Here, $\cS$ is the state space, $\cA$ is the finite action space,  $H$ is the length of each episode, $\reward_h: \cS \times \cA \rightarrow [0,1]$ is the reward function at stage $h$ and $\PP_h(s'|s,a) $ is the transition probability function at stage $h$ which denotes the probability for state $s$ to transfer to state $s'$ with action $a$ at stage $h$. A policy $\pi: \cS \times [H] \rightarrow \cA$ is a function which maps a state $s$ and the stage number $h$ to an action $a$. 
For any policy $\pi$ and stage $h\in [H]$, we define the action-value function $\qvalue_h^{\pi}(s,a)$ and value function $\vvalue_h^{\pi}(s)$ as follows
\begin{align}
\qvalue^{\pi}_h(s,a) &=\reward_h(s,a) + \EE\bigg[\sum_{h'=h+1}^H \reward_{h'}\big(s_{h'}, \pi(s_{h'},h')\big)\big| s_h=s,a_h=a\bigg],\ \vvalue_h^{\pi}(s) = \qvalue_h^{\pi}(s, \pi(s,h)),\notag
\end{align}
where $s_{h'+1}\sim \PP_h(\cdot|s_{h'},a_{h'})$. 
We define the optimal value function $V_h^*$ and the optimal action-value function $\qvalue_h^*$ as $V_h^*(s) = \max_{\pi}\vvalue_h^{\pi}(s)$ and $\qvalue_h^*(s,a) = \max_{\pi}\qvalue_h^{\pi}(s,a)$. By definition, the value function $\vvalue_h^{\pi}(s)$ and action-value function $\qvalue_h^{\pi}(s,a)$ are bounded in $[0,H]$. 
For any function $\vvalue: \cS \rightarrow \RR$, we denote $[\PP_h \vvalue](s,a)=\EE_{s' \sim \PP_h(\cdot|s,a)}\vvalue(s')$. Therefore, for each stage $h\in[H]$ and policy $\pi$, we have the following Bellman equation, as well as the Bellman optimality equation:
\begin{align}
    \qvalue_h^{\pi}(s,a) &= \reward_h(s,a) + [\PP_h\vvalue_{h+1}^{\pi}](s,a),\ \qvalue_h^{*}(s,a) = \reward_h(s,a) + [\PP_h\vvalue_{h+1}^{*}](s,a),\label{eq:bellman}
\end{align}
where $\vvalue^{\pi}_{H+1}=\vvalue^{*}_{H+1}=0$. At the beginning of the episode $k$, the agent determines a policy $\pi_k$ to be followed in this episode. At each stage $h\in[H]$, the agent observes the state $s_h^k$, chooses an action following the policy $\pi_k$ and observes the next state with $s_{h+1}^k \sim \PP_h(\cdot|s_h^k,a_h^k)$.

We consider linear function approximation in this work. Therefore, we make the following linear MDP assumption, which is firstly proposed in \cite{yang2019sample,jin2019provably}. 
\begin{assumption}\label{assumption-linear}
MDP $\cM(\cS, \cA, H, \{\reward_h\}_{h=1}^H, \{\PP_h\}_{h=1}^H)$ is a linear MDP such that for any stage $h\in[H]$, there exists an unknown vector $\bmu_h$, an unknown measure $\btheta_h(\cdot): \cS\rightarrow \RR^d$ and a known feature mapping $\bphi: \cS \times \cA \rightarrow \RR^d$, such that for each $(s,a) \in \cS \times \cA$ and $s' \in \cS$, 
\begin{align}
    \PP_h(s'|s,a) = \big\la \bphi(s,a), \btheta_h(s')\big\ra,\reward_h(s,a)=\big\la \bphi(s,a),\bmu_h\big\ra.\notag
\end{align}
\end{assumption}
For simplicity, we assume that $\bmu_h$, $\btheta_n(\cdot)$ and $\bphi(\cdot,\cdot)$ satisfy $\|\bphi(s,a)\|_2 \leq 1$ for all $s,a$, $\|\bmu_h\|_2 \leq \sqrt{d}$ and $\big\|\btheta_h(\cS)\big\|_2\leq \sqrt{d}$. The linear MDP assumption automatically suggests that for any policy $\pi$, the action-value function $Q^\pi_h$ is always a linear function of the given feature mapping $\bphi$, which is summarized in the following proposition. 
\begin{proposition}[Proposition 2.3, \cite{jin2019provably}]\label{prop:linearq}
For any policy $\pi$, there exist weights $\{\wb_h^\pi\}_{h=1}^H$ such that for any $s,a,h \in \cS \times \cA \times [H]$, $Q_h^\pi(s,a) = \la \bphi(s,a), \wb_h^\pi\ra$. 
\end{proposition}

Next we define the regret and $(\epsilon, \delta)$-PAC formally. 
\begin{definition}\label{def:Regret} 
For an RL algorithm $\alg$, we define its regret on learning an MDP $M(\cS, \cA, H, \reward, \PP)$ in the first $K$ episodes as the sum of the suboptimality for episode $k = 1,\ldots, K$, 
\begin{align}
    \text{Regret}(K) = \sum_{k=1}^K \vvalue_1^*(s_1^k) - \vvalue_1^{\pi_k}(s_1^k),\notag
\end{align}
where $\pi_k$ is the policy in the $k$-th episode .
\end{definition}
\begin{definition}\label{def:pac}
For an RL algorithm $\alg$  and a fixed $\epsilon$, let $\pi_1,\pi_2,\dots$ be the policies generated by $\alg$. Let $N_{\epsilon} = \sum_{k=1}^\infty \ind\{\vvalue_1^*(s_1^k) - \vvalue_1^{\pi_k}(s_1^k) > \epsilon\}$ be the number of episodes whose suboptimality gap is greater than $\epsilon$. Then we say $\alg$ is $(\epsilon, \delta)$-PAC with sample complexity $f(\epsilon, \delta)$ if 
\begin{align}
    \PP(N_{\epsilon} > f(\epsilon, \delta)) \leq \delta.\notag
\end{align}
\end{definition}

\begin{remark}
\citet{dann2017unifying} suggested that an algorithm with a sublinear regret is not necessarily to be an $(\epsilon,\delta)$-PAC algorithm. However, with some modification, \citet{jin2018q} showed that any algorithm with a sublinear regret can be converted to a \emph{new} algorithm which is $(\epsilon,\delta)$-PAC, which does not contradict with the claim by \cite{dann2017unifying}. For example, \citet{azar2011speedy} and \citet{zhang2020almost} proposed algorithms with $\tilde O(\sqrt{SAHT})$ regret, and both algorithms can be converted into new algorithms which are $(\epsilon,\delta)$-PAC with sample complexity $\tilde O(SAH^2/\epsilon^2)$.



\end{remark}

Both regret and PAC guarantees are not perfect. As \citet{dann2017unifying} showed, an algorithm with sub-linear regret or $(\delta,\epsilon)$-PAC bound may fail to converge to the optimal policy. For an $(\delta,\epsilon)$-PAC algorithm with $\Delta_t=\epsilon/2 (t\in \NN)$, it still has linear regret $O(\epsilon T)$ and will never converge to the optimal policy. For an algorithm with a sub-linear regret bound, a constant sub-optimality gap may still occur infinite times. Therefore, \citet{dann2017unifying} proposed uniform-PAC algorithms, which are defined formally as follows.
\begin{definition}\label{def:uni-pac}
For an RL algorithm $\alg$, let $\pi_1,\pi_2,\dots$ be the policies generated by $\alg$. Let $N_{\epsilon} = \sum_{k=1}^\infty \ind\{\vvalue_1^*(s_1^k) - \vvalue_1^{\pi_k}(s_1^k) > \epsilon\}$ be the number of episodes whose suboptimality gap is greater than $\epsilon$. We say $\alg$ is uniform-PAC for some $\delta \in (0,1)$  with sample complexity $f(\epsilon, \delta)$ if 
\begin{align}
    \PP(\exists\epsilon>0,\ N_{\epsilon} > f(\epsilon, \delta)) \leq \delta.\notag
\end{align}
\end{definition}


The following theorem suggests that a uniform-PAC algorithm is automatically a PAC algorithm and an algorithm with sublinear regret.

\begin{theorem}[Theorem 3, \citep{dann2017unifying}]\label{theorem: pac-transfer}
If an algorithm $\alg$ is uniform-PAC for some $\delta \ge 0$, with sample complexity $\tilde O(C_1/\epsilon+ C_2/\epsilon^2)$, where $C_1,C_2$ are constant and depend only on $S,A,H,\log(1/\delta)$. Then, we have the following results:
\begin{itemize}[leftmargin = *]
    \item 1: $\alg$ will converge to optimal policies with high probability at least $1-\delta$: $\PP\big(\lim_{k \rightarrow +\infty} \vvalue_1^*(s_1^k)-\vvalue_1^{\pi_k}(s_1^k)=0\big)\ge 1-\delta$ 
    \item 2: With probability at least $1-\delta$, for each $K\in \NN$, the regret for $\alg$ in the first $K$ episodes is upper bounded by $\tilde O(\sqrt{C_2K}+C_1+C_2)$.
    \item 3: For each $\epsilon\ge 0$, $\alg$ is also $(\epsilon,\delta)$-PAC with the same sample complexity $\tilde O(C_1/\epsilon+ C_2/\epsilon^2)$.
\end{itemize}
\end{theorem}
Theorem~\ref{theorem: pac-transfer} suggests that uniform-PAC is stronger than both the PAC and regret guarantees. In the remainder of this paper, we aim at developing uniform-PAC RL algorithms with linear function approximation.

\section{Warm up: Uniform-PAC Bounds for Linear Bandits}

To better illustrate the idea of our algorithm, in this section, we consider a contextual linear bandits problem, which can be regarded as a special linear MDP with $H = 1$. Let $\{\cD_k\}_{k=1}^\infty$ be a fixed sequence of decision/action sets. At round $k$, the agent selects an action $\xb_k \in \cD_k$ by the algorithm $\mathcal{H}$ and then observes the reward $r_k = \la \bmu^*, \xb_k\ra + \epsilon_k$, where $\bmu^* \in \RR^d$ is a vector unknown to the agent and $\epsilon_k$ is a sub-Gaussian random noise. $\xb_k,\epsilon_k,\bmu^*$ satisfy the following properties:
\begin{align}
    \forall k\in \NN, \lambda \in \RR,\  \EE\big[e^{\lambda\epsilon_k}|\xb_{1:k}, \epsilon_{1:k-1}\big] \leq \exp(\lambda^2/2), \|\xb_k\|_2 \leq 1, \|\bmu^*\|_2\leq 1. \label{def:eee}
\end{align}

Our goal is to design an $(\epsilon, \delta)$-uniform-PAC algorithm with sample complexity $f(\epsilon,\delta)$ such that 
\begin{align}
    \PP\bigg(\exists \epsilon >0,\ \sum_{k=1}^\infty\ind\bigg\{\Delta_k: = \max_{\xb \in \cD_k}\la \bmu^*, \xb\ra - \la \bmu^*, \xb_k\ra>\epsilon\bigg\} > f(\epsilon, \delta)\bigg) < \delta,\notag
\end{align}
where $\Delta_k: = \max_{\xb \in \cD_k}\la \bmu^*, \xb\ra - \la \bmu^*, \xb_k\ra$ denotes the suboptimality at round $k$.

\CCC{Here we assume the weight vector $\bmu^*$ satisfies $\|\bmu^*\|_2\leq 1$, to be consistent with the assumption made in \citet{abbasi2011improved}. 
Our assumption can be easily relaxed to the general $\|\bmu^*\|_2\leq B$ case  with an additional $\log B$ factor in the sample complexity, as can be seen in the following analysis.
} 
\begin{algorithm*}[t]
	\caption{Uniform-PAC OFUL (UPAC-OFUL)}\label{algorithm-bandit}
	\begin{algorithmic}[1]
    \REQUIRE Regularization parameter $\lambda$, confidence radius $\beta_l(l\in \NN)$
    \STATE Set $\cC^l \leftarrow \emptyset, l \in \NN$ and the total level $S_1=1$
	\FOR {round $k=1,2,..$}
	   
	        \FOR {all level $l\in[S_k]$}
	            \STATE Set $\bSigma_{k}^l=\lambda \Ib+\sum_{i\in \cC^l} \xb_i\xb_i^\top$\label{algorithm1-line4}
	            \STATE Set $\bbb_{k}^l=\sum_{i\in \cC^l} \xb_i\reward_i$ and $\wb_{k}^{l}= (\bSigma_{k}^l)^{-1}\bbb_{k}^l$\label{algorithm1-line5}
	        \ENDFOR

	    \STATE Receive the action set $\cD_k$
	    \STATE Choose action $\xb_k\leftarrow \argmax_{\xb \in \cD_k} \min_{1\leq l\leq S_k} (\wb_{k}^{l})^{\top}\xb+\beta_{l}\sqrt{\xb^{\top}(\bSigma_{k}^l)^{-1} \xb}$\label{algorithm1-line8}
	        \STATE Set level $l_k=1$ \label{algorithm1-line9}
	        \WHILE{$\sqrt{\xb_k^{\top}(\bSigma_{k}^{l_k})^{-1}\xb_k}\leq 2^{-{l_k}}$ and $l_k\leq S_k$} \label{algorithm1-line10}
	            \STATE $l_k\leftarrow l_k +1$ \label{algorithm1-line11}
	        \ENDWHILE
	        \STATE Add the new element $k$ to the set $\cC^{l_k}$ and receive the reward $\reward_k$ \label{algorithm1-line13}
	         \STATE Set the total level $S_{k+1}$ as $S_{k+1}=\max_{l:|\cC^l|>0}l$
	\ENDFOR

	\end{algorithmic}
\end{algorithm*}

\noindent\textbf{Why existing algorithms fail to be uniform-PAC?}
Before proposing our algorithm, it is natural to ask whether existing methods have already been uniform-PAC. We take OFUL \citep{abbasi2011improved} for example, which is the state-of-the-art linear bandit algorithm in our setting. At round $k$, OFUL constructs an optimistic estimation of the true linear function $\la \bmu^*, \xb\ra$, by doing linear regression over all past $k$ selected actions $\xb_i, 1 \leq i \leq k$ and their corresponding rewards. The optimistic estimation has a closed-form as the summation of the linear regression predictor and a quadratic confidence bound $\wb_k^{\top}\xb+\alpha\sqrt{\xb^{\top} \bSigma_k^{-1}\xb},$ where $\bSigma_k= \lambda \Ib+\sum_{i=1}^{k-1}\xb_i \xb_i^{\top}$  \citep{li2010contextual}. 
Following the standard analysis of OFUL in \cite{abbasi2011improved}, we obtain the following upper confidence bound of the suboptimality gap $\Delta_k$:
\begin{align}
    \text{With probability at least }1-\delta,\ \forall k>0,\ \Delta_k  = O\big(\sqrt{d \log (k/\delta)}\|\xb_k\|_{\bSigma_{k}^{-1}}\big),\label{xx}
\end{align}
where the $\log k$ is due to the fact that OFUL makes use of \emph{all} past $k$ observed actions. Since the agent can only say whether an arm is good or not based on the confidence bound of $\Delta_k$, due to the existence of the $\log k$ term in \eqref{xx}, the bounds on the suboptimality gap for the ``good'' arms may be large (since $\log k$ grows as $k$ increases). That makes the agent fail to recognize those ``good'' arms and instead pull the ``bad'' arms infinite times, which suggests that OFUL is not a uniform-PAC algorithm. For other algorithms, they either need to know the total round $T$ before running the algorithm \citep{chu2011contextual}, or need to assume that the decision sets $\cD_k$ are identical (e.g., algorithms for best-arm identification \cite{soare2014best}), thus none of them fits into our setting. 

\noindent\textbf{Key techniques of our algorithm.} In order to address the aforementioned issue, we proposed UPAC-OFUL in Algorithm \ref{algorithm-bandit}. The key idea of Algorithm \ref{algorithm-bandit} is to divide all the historical observed data into non-overlapping sets $\cC^l$, while each $\cC^l$ only includes \emph{finite} past historical observed actions. This helps successfully avoid the $\log k$ term appearing in the confidence bound in \cite{abbasi2011improved}. Then at round $k$, Algorithm \ref{algorithm-bandit} only constructs optimistic estimation of $\la \bmu^*, \xb_k\ra$ over the first $S_k$ sets $\cC^l$ individually, where $S_k$ is the number of non-empty sets $\cC^l$. In detail, the optimistic estimation over $\cC^l$ has the form
\begin{align}
    &(\wb_{k}^{l})^{\top}\xb+\beta_{l}\sqrt{\xb^{\top}(\bSigma_{k}^l)^{-1} \xb},\label{opt:bandit}
\end{align}
where $\bSigma_{k}^l$ is the covariance matrix for actions in set $\cC^l$ (Line \ref{algorithm1-line4}), and $\wb_{k}^{l}$ is the estimation of $\bmu^*$ obtained by ridge regression defined in Line \ref{algorithm1-line5}. Meanwhile, for a newly selected action $\xb_k$, Algorithm~\ref{algorithm-bandit} needs to decide which $\cC^l$ it should be added to. Inspired by \cite{chu2011contextual}, Algorithm \ref{algorithm-bandit} tests the ``uncertainty'' of $\xb_k$ against $\cC^l$, by calculating its confidence bound $\|\xb_k\|_{(\bSigma_{k}^{l})^{-1}}$. Then Algorithm \ref{algorithm-bandit} adds $\xb_k$ to the \emph{lowest possible} level where the ``uncertainty'' is larger than a certain threshold (i.e., $\|\xb_k\|_{(\bSigma_{k}^{l})^{-1}}>2^{-l}$). Such a selection rule guarantees two things simultaneously. First, it ensures that the cardinality of each $\cC^l$ is finite, due to the fact that the summation of $\|\xb_k\|_{(\bSigma_{k}^{l})^{-1}}$ can be properly bounded. Second, it also guarantees that the "uncertainty" of the reward corresponding to $\xb_k$ is still small, since by the level selection rule we have $\|\xb_k\|_{(\bSigma_{k}^{l-1})^{-1}}\leq 2^{-(l-1)}$. 
Lastly, to make use of all $S_k$ optimistic estimations, Algorithm \ref{algorithm-bandit} constructs the final predictor as the minimal value of $S_k$ individual predictor \eqref{opt:bandit} over $\cC^l$ (Line \ref{algorithm1-line8}). 
Since each individual predictor is a valid upper bound of the true function, the minimum of them is still valid and tighter than each of them (except the smallest one), which makes it possible to provide a stronger uniform-PAC guarantee.

The following theorem shows that Algorithm \ref{algorithm-bandit} is indeed uniform-PAC.

\begin{theorem}\label{thm-bandit}
For any $\delta \in (0,1)$, if we set $\lambda=1$ and $\beta_l=6\sqrt{dl\log (dl/\delta)}$ for every level $l\in \NN $, then there exists a constant $C$ such that with probability at least $1-\delta$, for all $\epsilon> 0$,  the number of rounds in Algorithm \ref{algorithm-bandit} which have sub-optimality no less than $\epsilon$ is bounded by
\begin{align}
     \sum_{k=1}^\infty\ind\bigg\{\max_{\xb \in \cD_k}\la \bmu^*, \xb\ra - \la \bmu^*, \xb_k\ra>\epsilon\bigg\} \leq  \frac{Cd^2\log^3\big(d/(\delta\epsilon)\big)}{\epsilon^2}.\notag
\end{align}
\end{theorem}
\begin{remark}
Theorem \ref{thm-bandit} suggests that Algorithm \ref{algorithm-bandit} is uniform-PAC with sample complexity $O(d^2\log^3\big(d/(\delta\epsilon)\big)/\epsilon^2)$. According to Theorem \ref{theorem: pac-transfer}, this new algorithm will converge to the optimal policy. Theorem~\ref{thm-bandit} also implies an $\tilde O(d\sqrt{T})$ regret for infinite-arm linear bandit problem. This result matches the lower bound $\Omega(d\sqrt{T})$ \citep{dani2008stochastic} up to a logarithmic factor. Furthermore, Theorem \ref{thm-bandit} implies that Algorithm~\ref{algorithm-bandit} is an $(\epsilon,\delta)$-PAC algorithm with sample complexity $\tilde O(d^2/\epsilon^2)$. Specifically, if we set $\epsilon=\Delta_{\min}$, Theorem~\ref{thm-bandit} implies an $\tilde O(d^2/\Delta_{\min}^2)$ sample complexity to identify the best arm\footnote{ \citet{soare2014best} denoted by $\Delta_{\min}$ the gap of the rewards between the best arm and the second-best arm. In this setting, the sample complexity to find the best arm is identical to the sample complexity to find an $\epsilon<\Delta_{\min}$ sub-optimal arm.}, 
which matches the sample complexity $\tilde O(d\log K/\Delta_{\min}^2)$ in \cite{soare2014best} when $K = \Theta(2^d)$. 

\end{remark}

\section{Uniform-PAC Bounds for Linear MDPs}
In this section, we propose our new $\algname$ algorithm (Algorithm \ref{algorithm-mdp}) for learning linear MDPs, and provide its theoretical guarantee.

\noindent\textbf{Intuition behind $\algname$}
At a high level, $\algname$ inherits the structure of Least-Square Value Iteration with UCB (LSVI-UCB)  proposed in \cite{jin2019provably}. The Bellman optimality equation gives us the following equation:
\begin{align}
    r_h(s,a) +  [\PP_h V_{h+1}^*](s,a)= Q_h^*(s,a) = \la \btheta_h^*, \bphi(s,a)\ra,\label{ii1}
\end{align}
where the second equality holds due to Proposition \ref{prop:linearq}. \eqref{ii1} suggests that in order to learn $Q_h^*$, it suffices to learn $\btheta_h^*$, which can be roughly regarded as the unknown vector of a linear bandits problem with actions $\bphi(s,a)$ and rewards $r_h(s,a) +  V_{h+1}^*(s')$, where $(s,a,s')$ belongs to some set $\cC$. Since $V_{h+1}^*$ is unknown, we use its estimation $V_{h+1}$ to replace it. Therefore, we can apply Algorithm \ref{algorithm-bandit} to this equivalent linear bandits problem to obtain our uniform-PAC RL algorithm $\algname$. 

\noindent\textbf{Details of $\algname$}
We now describe the details of $\algname$. For each stage $h$, $\algname$ maintains non-overlapping index set $\{\cC_h^l\}_l$, each $\cC_h^l$ contains state-action-next-state triples $(s_h^i, a_h^i, s_{h+1}^i)$. Let $S_1=1$ and $S_k$ denote the number of non-empty sets $\{\cC_1^l\}_l$ at episode $k$ for $k\ge 2$. Instead of maintaining only one estimated optimal value function $V_{k,h}$ and action-value function $Q_{k,h}$ \citep{jin2019provably}, $\algname$ maintains \emph{a group} of estimated value functions $\{V_{k,h}^l\}_l$ and action-value functions $\{Q_{k,h}^l\}_l$. In detail, at stage $h$, given $\{V_{k, h+1}^l\}_l$, $\algname$ calculates $\wb_{k,h}^l$ as the minimizer of the ridge regression problem with training dataset $(s_h^i, a_h^i, s_{h+1}^i) \in \cC_h^l$ and targets $V_{k, h+1}^l(s_{h+1}^i)$ (Line \ref{algorithm2-line9}), and defines $Q_{k,h}^l$ as the summation of the linear predictor $(\wb_{k,h}^{l})^{\top}\bphi(s,a)$ and a quadratic confidence bonus $\beta_{l}\sqrt{\bphi(s,a)^{\top}(\bSigma_{k,h}^l)^{-1} \bphi(s,a)}$ (Line \ref{algorithm2-line10}), where $\beta_l$ is the confidence radius  for level $l$ and $\bSigma_{k,h}^l$ is the covariance matrix for contexts in set $\cC_h^l$. 
Then $\algname$ defines value function $\vvalue_{k,h}^l$ as the maximum of the minimal value over the first $l$ action-value functions (Line \ref{algorithm2-line13}). The max-min structure is similar to its counterpart for linear bandits in Algorithm \ref{algorithm-bandit}, which provides a tighter estimation of the optimal value function and is pivotal to achieve uniform-PAC guarantee.  

After constructing action-value functions $\{Q_{k,h}^l\}_l$, $\algname$ executes the greedy policy induced by the minimal of action-value function $Q_{k,h}^l$ over $1\leq l \leq l_{h-1}^k-1$, where $l_0^k=S_k+1$ and $ l_{h-1}^k$ is the level of the set $\cA_{h-1}^l$ that we add the triple $(s_{h-1}^k, a_{h-1}^k, s_{h}^k)$. After obtaining $(s_h^k, a_h^k, s_{h+1}^k)$, to decide which set $\cC_h^l$ should this triple be added, $\algname$ calculates the confidence bonus $\sqrt{\bphi(s_h^k,a_h^k)^{\top}(\bSigma_{k,h}^{l})^{-1} \bphi(s_h^k,a_h^k)}$ and puts it into the $l$-th set if the confidence bonus is large (Line~\ref{algorithm2-line20}), similar to that of Algorithm \ref{algorithm-bandit}.




\begin{algorithm*}[t]
	\caption{Uniform PAC Least-Square Value Iteration ($\algname$)}\label{algorithm-mdp}
	\begin{algorithmic}[1]
    \REQUIRE Regualarization parameter $\lambda$, confidence radius $\beta_l$
    \STATE Set $\cC_h^l\leftarrow \emptyset, l \in \NN, h \in [H]$ and set the total level $S_1=1$
	\FOR {episode $k=1,2,..$}
	    \STATE Set $\vvalue_{k,H+1}^{l}(s,a)=0$ for all state-action pair $(s,a)\in \cS\times \cA$ and all level $l\in [S_k]$ 
	    \FOR {stage $h=H,H-1,..,1$}
	        \FOR {all level $l\in[S_k]$}
	            \STATE Set $\bSigma_{k,h}^l=\lambda \Ib+\sum_{i\in \cC_h^l} \bphi(s_h^i,a_h^i)\bphi(s_h^i,a_h^i)^\top$ \label{algorithm2-line 7}
	            \STATE Set $\bbb_{k,h}^l=\sum_{i\in \cC_h^l} \bphi(s_h^i,a_h^i)\Big[\reward_h(s_h^i,a_h^i)+\vvalue_{k,h+1}^l(s_{h+1}^i)\Big]$
	            \STATE $\wb_{k,h}^{l}\leftarrow (\bSigma_{k,h}^l)^{-1}\bbb_{k,h}^l$\label{algorithm2-line9}
	            \STATE $\qvalue_{k,h}^{l}(s,a)\leftarrow \min\Big\{H,(\wb_{k,h}^{l})^{\top}\bphi(s,a)+\beta_{l}\sqrt{\bphi(s,a)^{\top}(\bSigma_{k,h}^l)^{-1} \bphi(s,a)}\Big\}$\label{algorithm2-line10}
	        \ENDFOR
	        \FOR {all level $l\in[S_k]$}
            \STATE $\vvalue_{k,h}^{l}(s)\leftarrow\max_{a} \min_{1\leq i\leq l} \qvalue_{k,h}^{i}(s,a)$\label{algorithm2-line13}
	        \ENDFOR
	    \ENDFOR
	    \STATE Receive the initial state $s_{1}^k$ and  set the current level $l_0^k=S_k+1$
	    \FOR {stage $h=1,2,..,H$}
	        \STATE Take action $a_h^k\leftarrow \argmax_{a} \min_{1\leq i\leq l_{h-1}^k-1} \qvalue_{k,h}^{i}(s_h^k,a)$
	        \STATE Set level $l_h^k=1$
	        \WHILE{$\sqrt{\bphi(s_h^k,a_h^k)^{\top}(\bSigma_{k,h}^{l_h^k})^{-1} \bphi(s_h^k,a_h^k)}\leq 2^{-l_h^k}$ and $l_h^k\leq l_{h-1}^k-1$} \label{algorithm2-line20}
	            \STATE $l_h^k \leftarrow l_h^k+1$ \label{algorithm2-line21}
	        \ENDWHILE
	        \STATE Add element $k$ to the set $\cC_h^{l_h^k}$ \label{algorithm2-line23}
	        \STATE Receive the reward $\reward_h(s_h^k,a_h^k)$ and the next state $s_{h+1}^k$
	    \ENDFOR
	    \STATE Set the total level $S_{k+1}$ as $S_{k+1}=\max_{l:|\cC_1^l|>0}l$
	\ENDFOR

	\end{algorithmic}
\end{algorithm*}

The following theorem shows that $\algname$ is uniform-PAC for learning linear MDPs.
\begin{theorem}\label{thm:2}
Under Assumption \ref{assumption-linear}, there exists a positive constant $C$ such that for any $\delta \in(0,1)$, if we set $\lambda=1$ and $\beta_l=C dHl \sqrt{ \log(dlH/\delta)}$, then with probability at least $1-\delta$, for all $\epsilon > 0$, we have
\begin{align}
    \sum_{k=1}^{\infty} \ind\{\vvalue_{1}^{*}(s_1^k)-\vvalue_1^{\pi_k}
    \ge \epsilon\}
    &=O( d^3H^5 \log^4(dH/(\delta\epsilon))/\epsilon^2).\notag
\end{align}
\end{theorem}
\begin{remark}
Theorem \ref{thm:2} suggests that algorithm $\algname$ is uniform-PAC with sample complexity $O( d^3H^5 \log^4(dH/(\delta\epsilon))/\epsilon^2)$.
According to Theorem \ref{theorem: pac-transfer}, $\algname$ will converge to the optimal policy with high probability. Theorem \ref{thm:2} also implies an $\tilde O(\sqrt{d^3H^4T})$ regret for linear MDPs. This result matches the regret bound $\tilde O(\sqrt{d^3H^3T})$ of LSVI-UCB \cite{jin2019provably} up to a $\sqrt{H}$-factor. Furthermore, Theorem \ref{thm:2} also
implies $\algname$ is an $(\epsilon,\delta)$-PAC with sample complexity $\tilde O(d^3H^5/\epsilon^2)$, which matches the $O(d^3H^3/\epsilon^2)$ sample complexity of LSVI-UCB up to an $H$ factor.
\end{remark}
\paragraph{Computational complexity}
\CCC{As shown in \citet{jin2019provably}, the time complexity of LSVI-UCB is $O(d^2AHK^2)$.
Compared with the LSVI-UCB algorithm, Algorithm~\ref{algorithm-mdp} maintains non-overlapping index sets $\{\cC_h^l\}$ and computes the corresponding optimistic value function for each level $\ell$. Without further assumption on the norm of $\bphi(s,a)$, the number of different levels in the first $K$ episodes is at most $K$, which incurs an additional factor of $K$ in the computational complexity in the worst case. However, if we assume the norm of $\bphi(s,a)$ equals $1$, then the number of levels in the first $K$ episodes is $O(\log K)$. Thus, since the computational complexity of our algorithm at each level can be bounded by that of LSVI-UCB, the computational complexity of Algorithm \ref{algorithm-mdp} is $O(d^2AHK^2\log K)$ under the above assumption.
}

\section{Proof Outline}
In this section, we show the proof roadmap for Theorem \ref{thm:2}, which consists of three key steps.

\noindent\textbf{Step 1: Linear function approximates the optimal value function well}

We first show that with high probability, for each level $l$, our constructed linear function $\la \wb_{k,h}^l, \bphi(s,a)\ra$ is indeed a "good" estimation of the optimal action-value function $Q_h^*(s,a)$. 
By the uniform self-normalized concentration inequality over a specific function class, for any policy $\pi$, any level $l\in \NN$ and any state-action pair $(s,a) \in \cS \times \cA$, we have the following concentration property:
\begin{align}
    (\wb_{k,h}^l)^{\top}\bphi(s,a) - \qvalue_h^{\pi}(s,a) = \big[\PP_h(\vvalue_{k,h+1}^l - \vvalue_{h+1}^{\pi})\big](s,a) + \Delta,\notag
\end{align}
where $|\Delta| \leq \beta_l \sqrt{\bphi(s,a)^{\top}(\Sigma_{k,h}^l)^{-1}\bphi(s,a)}.$
Then, taking a backward induction for each stage $h\in[H]$, 
let $\Omega=\big\{\qvalue_{k,h}^l(s,a)\ge \qvalue_h^*(s,a), \vvalue_{k,h}^{l}(s) \ge \vvalue_h^*(s), \forall (s,a)\in \cS \times \cA, k,l \in \NN , h\in[H]\big\}$ denote the event that the estimated value function $\qvalue_{k,h}^l$ and $\vvalue_{k,h}^l$ upper bounds the optimal value function $\qvalue^*_h$ and $\vvalue^*_h$. We can show that event $\Omega$ holds with high probability. (More details can be found in Lemmas \ref{lemma:mdp-transition} and \ref{lemma:mdp-upper-confidence})

\noindent\textbf{Step 2: Approximation error decomposition}

On the event $\Omega$, the sub-optimality gap in round $k$ is upper bounded by the function value gap between our estimated function $Q_{k,1}^l$ and the value function of our policy $\pi_k$. 
\begin{align}
    \vvalue_1^*(s_1^k) - \vvalue_1^{\pi_k}(s_1^k)&\leq \max_{a} \min_{1\leq l \leq l_0^k-1} \qvalue_{k,1}^{l}(s_1^k,a)- \qvalue_1^{\pi_k}(s_1^k,a_1^k)\leq \qvalue_{k,1}^{l_1^k-1}(s_1^k,a_1^k)- \qvalue_1^{\pi_k}(s_1^k,a_1^k),\notag
\end{align}
From now we only focus on the function value gap for level $l_h^k$. Some elementary calculation gives us
\begin{align}
    &\qvalue_{k,h}^{l_h^k-1}(s_h^k,a_h^k)- \qvalue_h^{\pi_k}(s_h^k,a_h^k)\notag\\
    &\leq \underbrace{2\beta_{l_h^k-1}\sqrt{\bphi(s_h^k,a_h^k)^{\top}(\bSigma_{k,h}^{l_h^k-1})^{-1} \bphi(s_h^k,a_h^k)}}_{I_h^k} +\qvalue_{k,h+1}^{l_{h+1}^k-1}(s_{h+1}^k,a_{h+1}^k)- \qvalue_{h+1}^{\pi_k}(s_{h+1}^k,a_{h+1}^k)\notag\\
    &\qquad +\underbrace{\big[\PP_h(\vvalue_{k,h+1}^{l_h^k-1} - \vvalue_{h+1}^{\pi_k})\big](s_h^k,a_h^k)-\big(\vvalue_{k,h+1}^{l_h^k-1}(s_{h+1}^k) - \vvalue_{h+1}^{\pi_k}(s_{h+1}^k)\big)}_{\Delta_{k,h}}.\label{hh1}
\end{align}
Therefore, by telescoping \eqref{hh1} from stage $h = 1$ to $H$, we conclude that the sub-optimality gap $\vvalue_1^*(s_1^k) - \vvalue_1^{\pi_k}(s_1^k)$ is upper bounded by the summation of the bonus $I_h^k$ and $\Delta_{k,h}$. The summation of the bonus $\sum I_h^k$ is the dominating error term.  According to the rule of level $l$, if $k \in \cC_h^{l}$ at stage $h\in[H]$, then $I_h^k$ satisfies $\sqrt{\bphi(s_h^k,a_h^k)^{\top}(\bSigma_{k,h}^{l-1})^{-1} \bphi(s_h^k,a_h^k)}\leq 2^{-(l-1)}$. Furthermore, the number of elements added into set  $\cC_h^l$ can be upper bounded by $ |\cC_h^l|\leq 17dlH4^{l}$ (See Lemma \ref{lemma:mdp-levelsize}). Thus we can bound the summation of $I_h^k$. For $\sum \Delta_{k,h}$, it is worth noting that $\Delta_{k,h}$ forms a martingale difference sequence, therefore by the standard Azuma-Hoeffding inequality, $\sum \Delta_{k,h}$ can be bounded by some non-dominating terms. Both of these two bounds will be used in the next step. 

\noindent\textbf{Step 3: From upper confidence bonus to uniform-PAC sample complexity}

In Step 2 we have already bounded the sub-optimality gap by the summation of bonus terms. In this step, we show how to transform the gap into the final uniform-PAC sample complexity. Instead of studying any accuracy $\epsilon$ directly, we focus on a special case where $\epsilon=H/2^i (i\in \NN)$, which can be easily generalized to the general case. For each fixed $\epsilon=H/2^i (i\in \NN)$, let $\cK$ denote the set $\cK=\big\{k|\vvalue_{1}^{*}(s_1^k)-\vvalue_1^{\pi_k}
    \ge \epsilon\big\}$ and $m=|\cK|$. On the one hand, according to the definition of set $\cK$, the summation of regret in episode $k (k\in \cK)$ is lower bounded by $m\epsilon$. On the other hand, according to Step 2, the summation of sub-optimality gaps of episode $k (k\in \cK)$, is upper bound by
    \begin{align}
      \sum_{k\in \cK}[\vvalue_{1}^{*}(s_1^k)-\vvalue_1^{\pi_k}]\leq    \underbrace{\sum_{k\in \cK}\sum_{h=1}^H 2\beta_{l_h^k-1}2^{-(l_h^k-1)}}_{J_1}+\underbrace{\sum_{k\in \cK}\sum_{h=1}^H\Delta_{k,h}}_{J_2}\label{ss1}.
    \end{align}
To further bound $J_1$, we divide those episode-stage pairs $(k,h) \in \NN \times [H]$ into two categories: $\cS_1=\big\{(k,h)|2\beta_{l_h^k-1}2^{-(l_h^k-1)}\leq \epsilon/(2H)\big\}$ and $\cS_2=\big\{(k,h)|2\beta_{l_h^k-1}2^{-(l_h^k-1)}> \epsilon/(2H)\big\}$. For the first category $\cS_1$, the summation of terms $I_h^k$ in this category is upper bound by 
\begin{align}
    \sum_{k\in \cK}\sum_{h=1}^H  \ind \{(k,h) \in \cS_1 \}2\beta_{l_h^k-1}2^{-(l_h^k-1)}\leq \sum_{k\in \cK}\sum_{h=1}^H \frac{\epsilon}{2H}=\frac{m\epsilon}{2} .\label{eq:sk-1}
\end{align}
 For any episode-stage pair $(k,h)$ in the second category $\cS_2$, the level $l_h^k$ satisfies $2^{l_h^k}\leq \tilde O(dH^2 /\epsilon)$ due to the choice of $\beta_{l_h^k - 1}$. Suppose $l'$ is the maximum level that satisfies $2^{l}\leq \tilde O(dH^2 /\epsilon)$ and for each level $l\leq l'$, the cardinality of set $\cC_h^l$ can be upper bounded by $ |\cC_h^l|\leq 17dlH4^{l}$. Thus, the summation of terms $I_h^k$ in category $\cS_2$ is upper bound by 
 \begin{align}
     \sum_{k\in \cK}\sum_{h=1}^H  \ind \{(k,h) \in \cS_2 \}2\beta_{l_h^k-1}2^{-(l_h^k-1)}&\leq \sum_{k\in \cK}\sum_{h=1}^H  \sum_{l=1}^{l'} \ind \{l_h^k=l\}2\beta_{l-1}2^{-(l-1)}\notag\\
     &=\sum_{h=1}^H \sum_{l=1}^{l'} 2\beta_{l-1}2^{-(l-1)} \sum_{k\in \cK}\ind \{l_h^k=l\}\notag\\
     &\leq \sum_{h=1}^H \sum_{l=1}^{l'} 2\beta_{l-1}2^{-(l-1)} 17dlH4^l\notag\\
     &=\tilde O(d^3H^5/\epsilon).\label{eq:sk-2}
 \end{align}
Back to \eqref{ss1}, for the second term $J_2$, according to Azuma–Hoeffding inequality, it can be controlled by $\tilde O(H\sqrt{Hm})$. Therefore, combining \eqref{eq:sk-1}, \eqref{eq:sk-2} with the bound of $J_2$, we have
\begin{align}
    m\epsilon\leq  \sum_{k\in \cK}\vvalue_{1}^{*}(s_1^k)-\vvalue_1^{\pi_k} \leq m\epsilon/2+ \tilde O(d^3H^5/\epsilon) + \tilde O(H\sqrt{Hm}),\notag
\end{align}
and it implies that the number of episodes with a sub-optimality gap greater than $\epsilon$ is bounded by $ \tilde O(d^3H^5/\epsilon^2)$. This completes the proof. 

\section{Conclusion and Future Work}
In this work, we proposed two novel uniform-PAC algorithms for linear bandits and RL with linear function approximation, with the nearly state-of-the-art sample complexity. To the best of our knowledge, these are the very first results to show that linear bandits and RL with linear function approximation can also achieve uniform-PAC guarantees, similar to the tabular RL setting. We leave proving their corresponding lower bounds and proposing algorithms with near-optimal uniform-PAC sample complexity as future work. 

\section*{Acknowledgments and Disclosure of Funding}
We thank the anonymous reviewers for their helpful comments. 
Part of this work was done when JH, DZ and QG participated the Theory of Reinforcement Learning program at the Simons Institute for the Theory of Computing in Fall 2020. JH, DZ and QG are partially supported by the National Science Foundation CAREER Award 1906169, IIS-1904183, BIGDATA IIS-1855099 and AWS Machine Learning Research Award. The views and conclusions contained in this paper are those of the authors and should not be interpreted as representing any funding agencies.

\bibliographystyle{ims}
\bibliography{reference}

\section*{Checklist}


\begin{enumerate}

\item For all authors...
\begin{enumerate}
  \item Do the main claims made in the abstract and introduction accurately reflect the paper's contributions and scope?
    \answerYes
  \item Did you describe the limitations of your work?
    \answerYes
  \item Did you discuss any potential negative societal impacts of your work?
    \answerNA{} Our work studies the uniform-PAC bounds for RL with function approximation, which is a pure theoretical problem and does not have any negative social impact.
  \item Have you read the ethics review guidelines and ensured that your paper conforms to them?
    \answerYes
\end{enumerate}

\item If you are including theoretical results...
\begin{enumerate}
  \item Did you state the full set of assumptions of all theoretical results?
    \answerYes
	\item Did you include complete proofs of all theoretical results?
    \answerYes
\end{enumerate}

\item If you ran experiments...
\begin{enumerate}
  \item Did you include the code, data, and instructions needed to reproduce the main experimental results (either in the supplemental material or as a URL)?
    \answerNA{}
  \item Did you specify all the training details (e.g., data splits, hyperparameters, how they were chosen)?
    \answerNA{}
	\item Did you report error bars (e.g., with respect to the random seed after running experiments multiple times)?
   \answerNA{}
	\item Did you include the total amount of compute and the type of resources used (e.g., type of GPUs, internal cluster, or cloud provider)?
    \answerNA{}
\end{enumerate}

\item If you are using existing assets (e.g., code, data, models) or curating/releasing new assets...
\begin{enumerate}
  \item If your work uses existing assets, did you cite the creators?
    \answerNA{}
  \item Did you mention the license of the assets?
    \answerNA{}
  \item Did you include any new assets either in the supplemental material or as a URL?
    \answerNA{}
  \item Did you discuss whether and how consent was obtained from people whose data you're using/curating?
    \answerNA{}
  \item Did you discuss whether the data you are using/curating contains personally identifiable information or offensive content?
    \answerNA{}
\end{enumerate}

\item If you used crowdsourcing or conducted research with human subjects...
\begin{enumerate}
  \item Did you include the full text of instructions given to participants and screenshots, if applicable?
    \answerNA{}
  \item Did you describe any potential participant risks, with links to Institutional Review Board (IRB) approvals, if applicable?
    \answerNA{}
  \item Did you include the estimated hourly wage paid to participants and the total amount spent on participant compensation?
    \answerNA{}
\end{enumerate}

\end{enumerate}

\newpage

\appendix
\section{OFUL Algorithm is not Uniform-PAC}

In this section, we consider a variant of the OFUL algorithm \citep{abbasi2011improved}.
Then we will present a hard-to-learn linear bandit instance and show that the variant of OFUL algorithm cannot have the uniform-PAC guarantee for this instance.

In the original OFUL algorithm \citep{abbasi2011improved}, following their notation, the agent selects the action by $\xb_k = \argmax_{(\xb,\btheta) \in \mathcal{D}_{k}\times \Theta_{k-1}} \langle \xb, \btheta \rangle$. Here we consider a variant of OFUL, where the agent selects the action by $\xb_k = \argmax_{(\xb,\btheta) \in \mathcal{D}_{k}\times \Theta_{k-1}\cap B(1)} \langle \xb, \btheta \rangle$, where $B(1)$ is a unit ball centered at zero.

 We consider a special contextual linear bandit instance with dimension $d=2$, $\btheta^*=(0,1)$, and zero noise. The action set in the first $K$ ($K$ is an arbitrary parameter that can be chosen later) rounds is $\{(1,0),(-1,0)\}$ and the action set in the following $\log K$ rounds is $\{(0,1),(0,-1)\}$. So the reward in each step can only be $1$ or $-1$. The agent will randomly choose one action if both actions attain $\argmax_{(\xb,\btheta) \in \mathcal{D}_{k}\times  \Theta_{k-1}\cap B(1)} \langle \xb, \btheta \rangle$. We can show that, in the first $K$ round, the confidence radius increases since the determinant of the covariance matrix increases, and it will not provide any information about the second dimension of the vector $\btheta^*$ since the two actions are orthogonal to $\btheta^*=(0,1)$. After the first $K$ rounds, the confidence radius will be in the order of $\log K$, 
and the covariance matrix $\bSigma_K$ is a diagonal matrix and in the order of $\text{diag} (K,\log K)$. We can show that both $\btheta = (0,1)$ and $\btheta = (0,-1)$ belong to $\Theta_{k-1}\cap B(1)$, and thus attain the maximum of $\argmax_{(\xb,\btheta) \in \mathcal{D}_{k}\times \Theta_{k-1}\cap B(1)} \langle \xb, \btheta \rangle$. Therefore, the agent will almost ‘randomly’ pick one of the two actions in the later $\log K$ rounds. The random selection leads to a 1-suboptimality gap for about half of the $\log K$ rounds, which indicates that OFUL cannot be uniform-PAC for any finite $f(\epsilon, \delta)$ on this bandit problem, by selecting $\log K > f(\epsilon, \delta)$.

The above reasoning can be extended to the original OFUL algorithm with a more involved argument.

\section{Proof for Theorem \ref{thm-bandit}}\label{sec:bandit-lemma}
In this section, we provide the proof of Theorems \ref{thm-bandit} and for simplicity, let $\cC_{k}^l$ denote the index set $\cC^l$ at the beginning of round $k$. We first propose the following lemmas. 

\begin{lemma}\label{lemma:bandit-levelsize}
Suppose $\lambda\ge 1$, then for each level $l \in \NN$ and round $k \in \NN$, the number of elements in the index set $\cC_k^l$ is upper bounded by
\begin{align}
    |\cC_k^l|\leq 17dl4^{l}.\notag
\end{align}
\end{lemma}
\begin{proof}
See Appendix~\ref{proof:lemma:bandit-levelsize}.
\end{proof}
Lemma~\ref{lemma:bandit-levelsize} suggests that $\cC^l$ is always a finite set. 
\begin{lemma}\label{lemma:uniform-ucb-bandit}
 If we set $\lambda=1$ and $\beta_l=6\sqrt{dl\log (dl/\delta)}$ for every level $l\in \NN $, then with probability at least $1-\delta$, for all level $l\in \NN$ and all round $k\in \NN$, we have
        \begin{align}
        \|\wb_k^l-\bmu^*\|_{\bSigma_k^l}\leq \beta_{l}.\notag
    \end{align}

\end{lemma}
\begin{proof}
See Appendix~\ref{proof:lemma:uniform-ucb-bandit}.
\end{proof}
For simplicity, let $\cE$ denotes the event that the conclusion of Lemma \ref{lemma:uniform-ucb-bandit} holds. Therefore, Lemma \ref{lemma:uniform-ucb-bandit} suggests $ \Pr(\cE)\ge 1-\delta$.
\begin{proof}[Proof of Theorem \ref{thm-bandit}]
On the event $\cE$, for all level $l\in \NN$, round $k\in \NN$ and action $\xb \in \cD_k$, we have
\begin{align}
    (\wb_{k}^{l})^{\top}\xb+\beta_{l}\sqrt{\xb^{\top}(\bSigma_{k}^l)^{-1} \xb}-\la\bmu^*, \xb\ra &=(\wb_{k}^{l}-\bmu^*)^{\top}\xb+\beta_{l}\sqrt{\xb^{\top}(\bSigma_{k}^l)^{-1} \xb}\notag\\
    &\ge \beta_{l}\sqrt{\xb^{\top}(\bSigma_{k}^l)^{-1} \xb} -\|\wb_k^l-\bmu^*\|_{\bSigma_k^l}\|\xb\|_{(\bSigma_k^l)^{-1}}\notag\\
    &\ge \beta_l\sqrt{\xb^{\top}(\bSigma_{k}^l)^{-1}}-\beta_l \sqrt{\xb^{\top}(\bSigma_{k}^l)^{-1}} \notag\\
    &=0,\label{eq:8}
\end{align}
where the first inequality holds due to Cauchy-Schwarz inequality and the second inequality holds due to the definition of event $\cE$. \eqref{eq:8} implies that the estimated reward for each action $\xb\in \cD_k$ at level $l$: $(\wb_{k}^{l})^{\top}\xb+\beta_{l}\sqrt{\xb^{\top}(\bSigma_{k}^l)^{-1} \xb}$ is an upper confidence bound of the expected reward $\la\bmu^*, \xb\ra$. Thus, for each action $\xb \in \cD_k$, we have
\begin{align}
    \min_{1\leq l\leq S_k} (\wb_{k}^{l})^{\top}\xb+\beta_{l}\sqrt{\xb^{\top}(\bSigma_{k}^l)^{-1} \xb}\ge \min_{1\leq l\leq S_k}  \la\bmu^*, \xb\ra= \la\bmu^*, \xb\ra.\label{eq:9}
\end{align}
Therefore, for the sub-optimality gap at round $k$, we have
\begin{align}
     \max_{\xb \in \cD_k}\la \bmu^*, \xb\ra - \la \bmu^*,\xb_k\ra&\leq  \max_{\xb \in \cD_k} \min_{1\leq l\leq S_k} (\wb_{k}^{l})^{\top}\xb+\beta_{l}\sqrt{\xb^{\top}(\bSigma_{k}^l)^{-1} \xb}-  \la\bmu^*, \xb_k\ra\notag\\
    &= \min_{1\leq l\leq S_k} (\wb_{k}^{l})^{\top}\xb_k+\beta_{l}\sqrt{\xb_k^{\top}(\bSigma_{k}^l)^{-1} \xb_k}-  \la\bmu^*, \xb_k\ra,\label{eq:10}
\end{align}
where the first inequality holds due to \eqref{eq:9} and the second equality holds due to the policy in Algorithm \ref{algorithm-bandit} (line \ref{algorithm1-line8}).
 Thus, for each round $k\in \NN$, if the level $l_k>1$, we have 
\begin{align}
    \max_{\xb \in \cD_k}\la \bmu^*, \xb\ra - \la \bmu^*,\xb_k\ra&\leq  (\wb_{k}^{l_k-1})^{\top}\xb_k+\beta_{l_k-1}\sqrt{\xb_k^{\top}(\bSigma_{k}^{l_k-1})^{-1} \xb_k}-  \la\bmu^*, \xb_k\ra\notag\\
    &= (\wb_{k}^{l_k-1}-\bmu^*)^{\top}\xb_k+\beta_{l_k-1}\sqrt{\xb_k^{\top}(\bSigma_{k}^{l_k-1})^{-1} \xb_k}\notag\\
    &\leq \|\wb_k^{l_k-1}-\bmu^*\|_{\bSigma_k^{l_k-1}}\|\xb_k\|_{(\bSigma_k^{l_k-1})^{-1}}+\beta_{l_k-1}\sqrt{\xb_k^{\top}(\bSigma_{k}^{l_k-1})^{-1} \xb_k}\notag\\
    &\leq 2\beta_{l_k-1}  \sqrt{\xb_k^{\top}(\bSigma_{k}^{l_k-1})^{-1} \xb_k}\notag\\
    &\leq 2\beta_{l_k-1}\times 2^{-(l_k-1)},\notag
\end{align}
where the first inequality holds due to \eqref{eq:10} with the fact that $l_k-1\leq S_k$, the second inequality holds due to Cauchy-Schwarz inequality, the third inequality holds due to the definition of event $\cE$ and the last inequality holds due to the definition of level $l_k$ in Algorithm \ref{algorithm-bandit} (line \ref{algorithm1-line10} to line \ref{algorithm1-line11}). Since we set the parameter $\beta_l=6\sqrt{dl\log (dl/\delta)}$, there exists a large constant $C$ such that for any level $l$ satisfied $2^{l}\ge C\sqrt{d\log^2\big(d/(\delta\epsilon)\big)}/\epsilon$, we have $2\beta_{l-1}\times 2^{-(l-1)}\leq \epsilon$. For simplicity, we denote the minimum level $m=\bigg[\log\Big(C\sqrt{d\log^2\big(d/(\delta\epsilon)\big)}/\epsilon\Big)\bigg]$. Then for each round $k$, if level $l_k> m$, we have
\begin{align}
    \max_{\xb \in \cD_k}\la \bmu^*, \xb\ra - \la \bmu^*,\xb_k\ra&\leq 2\beta_{l_k-1}\times 2^{-(l_k-1)}\leq \epsilon.\label{eq:11}
\end{align}
Thus, for any $\epsilon>0$, we have
\begin{align}
    \sum_{k=1}^\infty\ind\bigg\{\max_{\xb \in \cD_k}\la \bmu^*, \xb\ra - \la \bmu^*, \xb_k\ra>\epsilon\bigg\} &\leq \sum_{k=1}^\infty\ind\big\{l_k\leq m\big\}\notag\\
    &= \sum_{k=1}^\infty\sum_{l=1}^m \ind\big\{l_k=l\big\}\notag\\
    &= \sum_{l=1}^m \sum_{k=1}^\infty\ind\big\{l_k=l\big\},\notag
\end{align}
where the inequality holds due to \eqref{eq:11}. According to Lemma \ref{lemma:bandit-levelsize}, the number of rounds with sub-optimality more than $\epsilon$ can be further bounded by
\begin{align}
\sum_{k=1}^\infty\ind\bigg\{\max_{\xb \in \cD_k}\la \bmu^*, \xb\ra - \la \bmu^*, \xb_k\ra>\epsilon\bigg\}&\leq
    \sum_{l=1}^m \sum_{k=1}^\infty\ind\big\{l_k=l\big\}
    \notag\\&\leq \sum_{l=1}^m 17dl4^l\notag\\&\leq C'd^2\log^3\big(d/(\delta\epsilon)\big)/\epsilon^2,\notag
\end{align}
where the second inequality holds due to Lemma \ref{lemma:bandit-levelsize} and the last inequality holds due to the definition of $m$ with the fact that $\sum_{l=1}^m l 4^l\leq m4^{m+1}$. Thus, we finish the proof of Theorem \ref{thm-bandit}.
\end{proof}

\section{Proof of Theorem \ref{thm:2}}\label{sec:mdp-lemma}
In this section, we provide the proof of Theorems \ref{thm:2} and for simplicity, let $\cC_{k,h}^l$ denote the index set $\cC_h^l$ at the beginning of episode $k$. We first propose the following lemmas.
\begin{lemma}\label{lemma:mdp-levelsize}
Suppose the parameter $\lambda$ satisfies $\lambda\ge 1$, then for each level $l\in \NN$ each stage $h\in[H]$ and each episode $k \in \NN$, the number of elements in the set $\cC_{k,h}^l$ is upper bounded by
\begin{align}
    |\cC_{k,h}^l|\leq 17dlh4^{l}.\notag
\end{align}
\end{lemma}
\begin{proof}
See Appendix~\ref{proof:lemma:mdp-levelsize}.
\end{proof}

\begin{lemma}\label{lemma:vector-range}
Under Assumption \ref{assumption-linear}, for each stage $h\in[H]$, each level $l\in \NN$ and each episode $k\in \NN$, the norm of weight vector $\wb_{k,h}^l$ can be upper bounded by
\begin{align}
    \|\wb_{k,h}^l\|_2\leq \frac{9d2^l \sqrt{H^3l}}{\sqrt{\lambda}}.\notag
\end{align}
\end{lemma}
\begin{proof}
See Appendix~\ref{proof:lemma:vector-range}.
\end{proof}

\begin{lemma}\label{lemma:mdp-concentration}
Suppose the parameter
$\lambda=1$, then there exists a large constant $C$, such that with probability $1-\delta/2$, for all stage $h\in[H]$, all episode $k\in \NN$ and all level $l \in \NN$, we have
\begin{align}
    \bigg\|\sum_{i\in \cC_{k,h}^l}\bphi(s_h^i,a_h^i)\big[\vvalue_{k,h+1}^l(s_{h+1}^i)-[\PP_h\vvalue_{k,h+1}^l](s_h^i,a_h^i) \big]\bigg\|_{(\bSigma_{k,h}^l)^{-1}}\leq C dHl\sqrt{ \log(dlH\beta_l^2/\delta)}.\notag
\end{align}
\end{lemma}
\begin{proof}
See Appendix~\ref{proof:lemma:mdp-concentration}.
\end{proof}

For simplicity, let $\cE$ denote the event that the conclusion of Lemma \ref{lemma:mdp-concentration} holds. Therefore, Lemma \ref{lemma:mdp-concentration} shows that $ \Pr(\cE)\ge 1-\delta/2$.

\begin{lemma}\label{lemma:mdp-transition}
Suppose $\lambda=1$ and $\beta_l=C dHl \sqrt{ \log(dlH/\delta)}$ with a large constant $C$, then on the event $\cE$, for all state-action pair $(s,a)\in \cS\times \cA$, stage $h\in[H]$, episode $k\in \NN$, level $l\in \NN$ and any policy $\pi$, we have 
\begin{align}
    (\wb_{k,h}^l)^{\top}\bphi(s,a) - \qvalue_h^{\pi}(s,a) = \big[\PP_h(\vvalue_{k,h+1}^l - \vvalue_{h+1}^{\pi})\big](s,a) + \Delta,\notag
\end{align}
where $|\Delta| \leq \beta_l \sqrt{\bphi(s,a)^{\top}(\Sigma_{k,h}^l)^{-1}\bphi(s,a)}.$
\end{lemma}
\begin{proof}
See Appendix~\ref{proof:lemma:mdp-transition}.
\end{proof}

\begin{lemma}\label{lemma:mdp-upper-confidence}
On the event $\cE$, for all state-action pair $(s,a)\in \cS\times \cA$, stage $h\in[H]$, episode $k\in \NN$ and level $l\in \NN$, we have
\begin{align}
    \qvalue_{k,h}^l(s,a)\ge \qvalue_{h}^{*}(s,a), \vvalue_{k,h}^l(s)\ge \vvalue_h^*(s).\notag
\end{align}
\end{lemma}
\begin{proof}
See Appendix~\ref{proof:lemma:mdp-upper-confidence}.
\end{proof}

Now we begin to prove Theorem \ref{thm:2}.
\begin{proof}[Proof of Theorem \ref{thm:2}]
Firstly, we focus on the special case that $\epsilon=H 2^{-i} (i\in \NN)$. Since $\epsilon=H 2^{-i}$ and we set the parameter $\beta_l=C dHl \sqrt{ \log(dlH/\delta)}$, there exists a large constant $C'$ such that for any level $l$ satisfied $2^{l}\ge C' dH^2 \log^{1.5}(dH/(\delta\epsilon))/\epsilon$, we have
\begin{align}
    4\beta_{l-1}2^{-l}= C dHl\sqrt{ \log(dlH\beta_{l-1}^2/\delta)}2^{-l} \leq \epsilon/(2H).\notag
\end{align}
For simplicity, we denote the maximum level $l'$ as $l'=\bigg[\log\Big(C' dH^2 \log^{1.5}(dH/(\delta\epsilon))/\epsilon\Big)\bigg]$.

 Now, let $k_0=0$, and for each $i\in \NN$, we denote $k_i$ as the minimum index of the episode where the sub-optimality gap is more than $\epsilon$, such that
\begin{align}
    k_i&=\min \Big\{k: k>k_{i-1}, \vvalue_1^*(s_1^k)-\vvalue_1^{\pi_k}(s_1^k)\ge \epsilon\Big\}.\label{eq:ti}
\end{align}
Now, we denote the set $K=\{k_i : i\in \NN,k_i< +\infty\}$ and we assume $K=\{k_1,..,k_m\}$. According to the definition of $k_i$ in \eqref{eq:ti}, we have
\begin{align}
    \sum_{i=1}^m \vvalue_1^*(s_1^{k_i}) - \vvalue_1^{\pi_{k_i}}(s_1^{k_i})\ge m \epsilon. \label{eq:21} 
\end{align}
On the other hand, for each episode $k \in \NN$ with total level $S_k$ at the beginning of episode $k$, we have
\begin{align}
    \vvalue_1^*(s_1^k) - \vvalue_1^{\pi_k}(s_1^k)&= \max_{a} \qvalue_1^*(s_1^k,a) - \qvalue_1^{\pi_k}(s_1^k,a_1^k)\notag\\
    &\leq \max_{a} \min_{1\leq l \leq S_k} \qvalue_{k,1}^{l}(s_1^k,a)- \qvalue_1^{\pi_k}(s_1^k,a_1^k)\notag\\
    &= \min_{1\leq l \leq S_k}\qvalue_{k,1}^{l}(s_1^k,a_1^k)- \qvalue_1^{\pi_k}(s_1^k,a_1^k)\notag\\
    &\leq \qvalue_{k,1}^{l_1^k-1}(s_1^k,a_1^k)- \qvalue_1^{\pi_k}(s_1^k,a_1^k), \label{eq:22}
\end{align}
where the first inequality holds due to Lemma \ref{lemma:mdp-upper-confidence}, the third equation holds due to the policy in Algorithm \ref{algorithm-mdp}  and the last inequality holds due to the fact that $l_1^k-1\leq S_k$. Furthermore, for each stage $h\in[H]$ and each episode $k \in \NN$, we have
\begin{align}
    &\qvalue_{k,h}^{l_h^k-1}(s_h^k,a_h^k)- \qvalue_h^{\pi_k}(s_h^k,a_h^k)\notag\\
    &\leq (\wb_{k,h}^{l_h^k-1})^{\top}\bphi(s_h^k,a_h^k)+\beta_{l_h^k-1}\sqrt{\bphi(s_h^k,a_h^k)^{\top}(\bSigma_{k,h}^{l_h^k-1})^{-1} \bphi(s,a)} - \qvalue_h^{\pi_k}(s_h^k,a_h^k)\notag\\
    &\leq 2\beta_{l_h^k-1}\sqrt{\bphi(s_h^k,a_h^k)^{\top}(\bSigma_{k,h}^{l_h^k-1})^{-1} \bphi(s_h^k,a_h^k)} +\big[\PP_h(\vvalue_{k,h+1}^l - \vvalue_{h+1}^{\pi_k})\big](s_h^k,a_h^k)\notag\\
    &\leq 2\beta_{l_h^k-1} 2^{-(l_h^k-1)}+\big[\PP_h(\vvalue_{k,h+1}^{l_h^k-1} - \vvalue_{h+1}^{\pi_k})\big](s_h^k,a_h^k)\notag\\
    &= 4\beta_{l_h^k-1} 2^{-l_h^k} + \underbrace{\big[\PP_h(\vvalue_{k,h+1}^{l_h^k-1} - \vvalue_{h+1}^{\pi_k})\big](s_h^k,a_h^k)-\big(\vvalue_{k,h+1}^{l_h^k-1}(s_{h+1}^k) - \vvalue_{h+1}^{\pi_k}(s_{h+1}^k)\big)}_{\Delta_{k,h}}\notag\\
    &\qquad+\vvalue_{k,h+1}^{l_{h}^k-1}(s_{h+1}^k) - \vvalue_{h+1}^{\pi_k}(s_{h+1}^k), \label{eq:23}
\end{align}
where the first inequality holds due to the definition of value function $\qvalue_{k,h}^l$ in Algorithm \ref{algorithm-mdp}, the second inequality holds due to Lemma \ref{lemma:mdp-transition} and the last inequality holds due to the definition of level $l_h^k$ in Algorithm \ref{algorithm-mdp}. Furthermore, for the term $\vvalue_{k,h+1}^{l_{h}^k-1}(s_{h+1}^k)$, it can be upper bounded by
\begin{align}
    \vvalue_{k,h+1}^{l_{h}^k-1}(s_{h+1}^k)&= \max_a\min_{1\leq l \leq l_{h}^k-1} \qvalue_{k,h+1}^{l}(s_{h+1}^k,a)\notag\\
    &= \min_{1\leq l \leq l_{h}^k-1} \qvalue_{k,h+1}^{l}(s_{h+1}^k,a_{h+1}^k)\notag\\
    &\leq  \qvalue_{k,h+1}^{l_{h+1}^k-1}(s_{h+1}^k,a_{h+1}^k),\label{eq:200}
\end{align}
where the inequality holds due to the fact that $l_{h+1}^k-1\leq l_{h}^k-1$.Substituting \eqref{eq:200} in to \eqref{eq:23} and taking a summation of \eqref{eq:23} with all stage $h\in[H]$, we have
\begin{align}
    \vvalue_1^*(s_1^k) - \vvalue_1^{\pi_k}(s_1^k)&\leq \qvalue_{k,1}^{l_1^k-1}(s_1^k,a_1^k)- \qvalue_1^{\pi_k}(s_1^k,a_1^k)
    \leq \sum_{h=1}^H 4\beta_{l_h^k-1} 2^{-l_h^k}+\sum_{h=1}^H \Delta_{k,h}.\label{eq:24}
\end{align}
Taking a summation of \eqref{eq:24} over all episode $k_i \in K$, we have
\begin{align}
    \sum_{i=1}^m \vvalue_1^*(s_1^{k_i}) - \vvalue_1^{\pi_{k_i}}(s_1^{k_i}) \leq \underbrace{\sum_{i=1}^ m\sum_{h=1}^H 4\beta_{l_h^{k_i}-1} 2^{-l_h^{k_i}}}_{I_1}+\underbrace{\sum_{i=1}^ m\sum_{h=1}^H\Delta_{k_i,h}}_{I_2}.\label{eq:25}
\end{align}
Since $4\beta_{l-1}2^{-l}\leq \epsilon/(2H)$ holds for all level $l> l'$, the term $I_1$ can be upper bounded by
\begin{align}
I_1&=\sum_{i=1}^ m\sum_{h=1}^H 4\beta_{l_h^{k_i}-1} 2^{-l_h^{k_i}}\notag\\ 
&\leq \sum_{i=1}^ m\sum_{h=1}^H \Big(\ind\{l_h^{k_i}\leq l'\}4\beta_{l_h^{k_i}-1} 2^{-l_h^{k_i}} +\frac{\epsilon}{2H}\Big) \notag\\
&= \sum_{i=1}^m \sum_{h=1}^H \Big(\sum_{l=1}^{l'} \ind\{l_h^{k_i}=l\}4\beta_{l-1}2^{-l} +\frac{\epsilon}{2H}\Big)\notag\\
&= \frac{m\epsilon}2 +\sum_{h=1}^H\sum_{l=1}^{l'} 4\beta_{l-1} 2^{-l} \sum_{i=1}^m \ind\{l_h^{k_i}=l\}.\notag
\end{align}
According to Lemma \ref{lemma:mdp-levelsize}, the number of elements added into the set $\cC_h^l$ is upper bounded by $ |\cC_h^l|\leq 17dlh4^{l}$ and it implies that $\sum_{i=1}^m \ind\{l_h^{k_i}=l\}\leq \sum_{k=1}^{+\infty} \ind\{l_h^{k}=l\}\leq 17dlh4^{l}$. Thus, we have
\begin{align}
   I_1 &\leq \frac{m\epsilon}2 +\sum_{h=1}^H\sum_{l=1}^{l'} 4\beta_{l-1} 2^{-l} \sum_{i=1}^m \ind\{l_h^{k_i}=l\}\notag\\
&\leq \frac{m\epsilon}2 + \sum_{h=1}^H\sum_{l=1}^{l'} 4\beta_{l-1} 2^{-l} \times 17dlh4^l\notag\\
&\leq \frac{m\epsilon}2 + 136\beta_{l'-1}2^{l'} d{l'} H^2.\label{eq:300}
\end{align}
According to the definition of level $l'$ and parameter $\beta_{l'}$, there exist a large constant $C''$ such that 
$I_1\leq m\epsilon/2+C'' d^3H^5 \log^4(dH/(\delta\epsilon))/\epsilon$.

For the term $I_2$, according to Lemma \ref{lemma:azuma}, for any fixed number $n\in \NN$ and $\epsilon=H/2^i$, with probability at least $1-\delta/\big(2i(i+1)n(n+1)\big)$, we have
\begin{align}
    \sum_{i=1}^ n\sum_{h=1}^H\Delta_{k_i,h}\leq H\sqrt{2Hn\log \frac{2i(i+1)n(n+1)}{\delta}}.\notag
\end{align}
Therefore, taking a union bound,  
with probability at least $1-\delta/\big(2i(i+1)\big)$, for all $n\in \NN$, we have
\begin{align}
    \sum_{i=1}^ n\sum_{h=1}^H\Delta_{k_i,h}\leq H\sqrt{2Hn\log \frac{2i(i+1)n(n+1)}{\delta}}.\notag
\end{align}
Thus, for the term $I_2$ and $\epsilon=H/2^i$, with probability at least $1-\delta/\big(2i(i+1)\big)$, we have
\begin{align}
    I_2=\sum_{i=1}^ m\sum_{h=1}^H\Delta_{k_i,h}\leq H\sqrt{2Hm\log \frac{2i(i+1)m(m+1)}{\delta}}.\label{eq:26}
\end{align}
Substituting \eqref{eq:300} and \eqref{eq:26} into \eqref{eq:25}, for $\epsilon=H/2^i$, we have
\begin{align}
    m\epsilon &\ge \sum_{i=1}^m \vvalue_1^*(s_1^{k_i}) - \vvalue_1^{\pi_{k_i}}(s_1^{k_i}) \notag\\
    &\ge \frac{m\epsilon}{2}+C'' d^3H^5 \log^4(dH/(\delta\epsilon))/\epsilon\notag\\
    &\qquad +H\sqrt{2m\log \frac{2i(i+1)m(m+1)}{\delta}},\notag
\end{align}
which implies $m\leq  O\Big( d^3H^5 \log^4\big(dH/(\delta\epsilon)\big)/\epsilon^2\Big)$. Finally, taking an union bound with the event $\cE$ and \eqref{eq:26}, with probability at least $1-\delta/2-\sum_{i=1}^{\infty}\delta/\big(2i(i+1)\big)=1-\delta$, for all $\epsilon=H/2^i(i\in \NN)$, we have
\begin{align}
    \sum_{k=1}^{\infty} \ind\{\vvalue_{1}^{*}(s_1^k)-\vvalue_1^{\pi_k}(s_1^k)\ge \epsilon\}\leq  O\Big( d^3H^5 \log^4\big(dH/(\delta\epsilon)\big)/\epsilon^2\Big).\notag
\end{align}
Finally, we extend the result to general $\epsilon>0$.  
For any $ H/2^i\leq \epsilon\leq H/2^{i-1}$, we have 
\begin{align}
    \sum_{k=1}^{\infty} \ind\{\vvalue_{1}^{*}(s_1^k)-\vvalue_1^{\pi_k}(s_1^k)
    \ge \epsilon\}&\leq \sum_{k=1}^{\infty} \ind\{\vvalue_{1}^{*}(s_1^k)-\vvalue_1^{\pi_k}(s_1^k)
    \ge H/2^i\}\notag\\
    &\leq  O\Big( d^3H^5 \log^4\big(dH/(\delta\epsilon')\big)/(H/2^i)^2\Big)\notag\\
    &=O\Big( d^3H^5 \log^4\big(dH/(\delta\epsilon)\big)/\epsilon^2\Big).\notag
\end{align}
Thus, we finish the proof of Theorem \ref{thm:2}.
\end{proof}

\section{Proof of Lemma in Section \ref{sec:bandit-lemma}}
\subsection{Proof of Lemma \ref{lemma:bandit-levelsize}}\label{proof:lemma:bandit-levelsize}

\begin{lemma}[Lemma 11, \citep{abbasi2011improved}]\label{lemma:sum-bonus}
For any vector sequence $\{\xb_k\}_{k=1}^K$ in $\RR^d$, We denote $\bSigma_0 = \lambda \Ib$ and $\bSigma_k = \bSigma_0 + \sum_{i=1}^{k}\xb_i\xb_i^\top$. If  $\lambda\ge \max(1,L^2)$ and $\|\xb_k\|_2 \leq L$ holds for all $k\in[K]$, then we have
\begin{align}
    \sum_{k=1}^K  \|\xb_k\|_{\bSigma_{k-1}^{-1}}^2 \leq 2d\log\frac{d\lambda+KL^2}{d \lambda}.\notag
\end{align}
\end{lemma}

\begin{proof}[Proof of Lemma \ref{lemma:bandit-levelsize}]
 We focus on round $k$ and we suppose set $\cC_k^l=\{k_1,..,k_m\}$ at that time, where $1\leq k_1<k_2<..<k_m< k$. According to the update rule of set $\cC^l$ in Algorithm \ref{algorithm-bandit} (line \ref{algorithm1-line9} to line \ref{algorithm1-line13}), for each $2\leq i\leq m$, we have $S_{k_i}\ge l$ and it implies that 
\begin{align}
   \xb_{k_i}^{\top}(\bSigma_{k_i}^{l})^{-1}\xb_{k_i}\ge 4^{-l},\label{eq:1}
\end{align}
where $\bSigma_{k_i}^l=\lambda \Ib+\sum_{j=1}^{i-1} \xb_{k_j}\xb_{k_j}^\top$. Therefore, taking a summation for \eqref{eq:1} over all $2\leq i\leq m$, we have
\begin{align}
    \sum_{i=1}^m\xb_{k_i}^{\top}(\bSigma_{k_i}^l)^{-1}\xb_{k_i}\ge \sum_{i=2}^m\xb_{k_i}^{\top}(\bSigma_{k_i}^l)^{-1}\xb_{k_i}\ge(m-1)4^{-l},\label{eq:2}
\end{align}
where the first inequality holds due to $\xb_{k_1}^{\top}(\bSigma_{k_1}^l)^{-1}\xb_{k_1}\ge 0$ and the second inequality holds due to \eqref{eq:1}.
On the other hand, according to Lemma \ref{lemma:sum-bonus}, this summation is upper bounded by
\begin{align}
    \sum_{i=1}^m\xb_{k_i}^{\top}(\bSigma_{k_i}^l)^{-1}\xb_{k_i}\leq 2d\log\frac{d\lambda+m}{d \lambda}\leq 2d \log (1+m/d),\label{eq:3}
\end{align}
where the first inequality holds due to Lemma \ref{lemma:sum-bonus} with $\|\xb_{k_i}\|_2\leq 1$ and the second inequality holds due to $\lambda \ge 1$.
Combining \eqref{eq:2} and \eqref{eq:3}, we have
\begin{align}
   (m-1)4^{-l} \leq 2d \log (1+m/d),\notag
\end{align}
which implies that the size of set $ |\cC_{k}^l|$ is upper bounded by $|\cC_{k}^l|=m\leq 17dl4^{l}$ for each $k\in \NN$. Therefore, we finish the proof of Lemma \ref{lemma:bandit-levelsize}.
\end{proof}

\subsection{Proof of Lemma \ref{lemma:uniform-ucb-bandit}}\label{proof:lemma:uniform-ucb-bandit}
\begin{lemma}[Theorem 2, \citep{abbasi2011improved}]\label{lemma:uniconvege}

Let $\{\epsilon_t\}_{t=1}^{\infty}$ be a real-valued stochastic process with corresponding filtration $ \{\mathcal{F}_t\}_{t=0}^{\infty}$ such that $\epsilon_t$ is $\mathcal{F}_t$-measure and $\epsilon_t$ is conditionally
$R$-sub-Gaussian, $i.e.$
\begin{align}
    \forall \lambda \in \RR, \EE[e^{\lambda\epsilon_t}|\mathcal{F}_{t-1}]\leq \exp\bigg(\frac{\lambda^2R^2}{2}\bigg)\notag.
\end{align}

Let $\{\xb_t\}_{t=1}^{\infty}$ be an $\RR^d$-valued stochastic process where $\xb_t$ is $\mathcal{F}_{t-1}$-measurable and we define $y_t=\la \xb_t,\btheta^* \ra+\epsilon_t$. With this notation, for any $t\ge 0$, we further define
\begin{align}
    \bSigma_t=\lambda I+\sum_{i=1}^{t}\xb_t\xb_t^{\top}, \bbb_t=\sum_{i=1}^{t}\xb_ty_t, \wb_t=(\bSigma_t)^{-1}\bbb_t.\notag
\end{align}
If we assume $\|\btheta^*\|\leq S$ and $\|\xb_t\|\leq L$ holds for all $t\in \NN$, then with probability at least $1-\delta$, for all $t\ge 0$, we have
\begin{align}
    \|\btheta^*-\wb_t\|_{\bSigma_t}\leq R\sqrt{d\log\bigg(\frac{1+tL^2/\lambda}{\delta}\bigg)}+\sqrt{\lambda}S.\notag
\end{align}
\end{lemma}

\begin{proof}[Proof of Lemma \ref{lemma:uniform-ucb-bandit}]
In this proof, we first focus on a fixed level $l \in \NN$ and then turn back to all level $l$. For a fixed level $l\in[N]$, we denote $k_0=0$, and for $i\in \NN$, we denote $k_i$ as the minimum index of the round where the action is added to the set $\cC^l$:
\begin{align}
    k_i&=\min \big\{k: k>k_{i-1}, l_{k}=l\big\}.\label{eq:filtration-bandit}
\end{align}
Under this notation, for all round $k (k_i<k\leq k_{i+1})$, we have 
\begin{align}
  \bSigma_{k}^l=\lambda \Ib+\sum_{j=1}^{i} \xb_{k_j}\xb_{k_j}^{\top}, \bbb_k^l=\sum_{j=1}^{i} \xb_{k_j}\reward_{k_j}, \wb_{k}^l=(\bSigma_{k}^l)^{-1} \bbb_k^l.\label{eq:4}
\end{align}
Now, we consider the $\sigma$-algebra filtration $\mathcal{F}_i=\sigma(\xb_1,..,\xb_{k_{i+1}},\reward_1,..,\reward_{k_{i+1}-1})$ that contains all randomness before receiving the reward $\reward_{k_{i+1}}$ at round $k_{i+1}$. By the definition of $\mathcal{F}_{i-1}$,  vector $\xb_{k_{i}}$ is $\mathcal{F}_{i-1}$-measurable and the noise $\epsilon_{k_{i}}=\reward_{k_{i}}-\la \xb_{k_{i}},\bmu^*\ra$ is $\mathcal{F}_{i}$-measurable.
 Since we choose the level $l_k$ and add element $k$ to the corresponding set $\cC^{l_k}$ before receiving the reward $\reward_k$ at round $k$, the noise $\epsilon_{k_{i}}$ is conditionally 1-Sub-Gaussian. According to Lemma \ref{lemma:uniconvege}, with probability at least $1-\delta\big(l(l+1)\big)$, for all $i\ge0$, we have
\begin{align}
    \|\bmu^*-\wb_{k_{i+1}}^l\|_{\bSigma_{k_{i+1}}^l}\leq  \sqrt{d\log\bigg(\frac{i+1}{\delta/\big(l(l+1)\big)}\bigg)}+1.\label{eq:5}
\end{align}
Combining \eqref{eq:4} and \eqref{eq:5}, for all round $k (k_i<k\leq k_{i+1})$, we have 
\begin{align}
    \|\bmu^*-\wb_k^l\|_{\bSigma_{k}^l}\leq \sqrt{d\log\bigg(\frac{i+1}{\delta/\big(l(l+1)\big)}\bigg)}+1.\label{eq:6}
\end{align}
Furthermore, Lemma \ref{lemma:bandit-levelsize} suggests that the size of set $|\cC^l|$ is upper bounded by $|\cC^l|\leq 17dl4^l$, which implies that $k_{17dl4^l+1}=+\infty$. Thus, \eqref{eq:6} implies that with probability at least $1-\delta/\big(l(l+1)\big)$, for all round $k\in \NN$, we have
\begin{align}
      \|\bmu^*-\wb_k^l\|_{\bSigma_{k}^l}\leq \sqrt{d\log\bigg(\frac{17dl4^l+1}{\delta/\big(l(l+1)\big)}\bigg)}+1\leq \beta_l.\label{eq:7}
\end{align}
Finally, taking a union bound for \eqref{eq:7} over all level $l\in \NN$, with probability at least $1-\sum_{l=1}^{\infty}\Big(\delta/\big(l(l+1)\big)\Big)=1-\delta$,  for all level $l\in[N]$ and all round $k\in \NN$, we have
        \begin{align}
        \|\wb_k^l-\bmu^*\|_{\bSigma_k^l}\leq \beta_{l}.\notag
    \end{align}
Thus, we finish the proof of Lemma \ref{lemma:uniform-ucb-bandit}
\end{proof}

\section{Proof of Lemma in Section \ref{sec:mdp-lemma}}
\subsection{Proof of Lemma \ref{lemma:mdp-levelsize}}\label{proof:lemma:mdp-levelsize}

\begin{proof}[Proof of Lemma \ref{lemma:mdp-levelsize}]
Similar to the proof of Lemma \ref{lemma:bandit-levelsize}, 
we focus on episode $k$ and we suppose set $\cC_{k,h}^l=\{k_1,..,k_m\}$ at that time, where $1\leq k_1<k_2<..<k_m< k$. For simplicity, we further define the auxiliary sets $\cB^l_{k,h}$ as
\begin{align}
    \cB^l_{k,h}=\big\{ \,i | \,1\leq i<k; l_h^i=l; (h=1 \text{ or } l_h^i<l_{h-1}^{i})\big\}.\notag
\end{align}
Notice that for each stage $h\ge 2$ and episode $i\in [k]$, there are two stopping rules for the while loop in Algorithm \ref{algorithm-mdp} (line \ref{algorithm2-line20} to line \ref{algorithm2-line21}) and $\cB^l_{k,h}$ consists of all episode $i\in [k]$ that stop with the first rule. Furthermore, for all element $k_i\in \cC_{k,h}^l$ with the second rule stopping rule, we have $l_{h-1}^k=l_{h}^k=l$ and it implies that $k_i \in \cC_{k,h-1}^l$. Combining these two cases, we have $\cC_h^l\subseteq \cB_{k,h}^l\cup \cC_{k,h-1}^l$ and it implies that
\begin{align}
    m=|\cC_{k,h}^l|\leq \sum_{j=1}^h |\cB_{k,j}^l|,\label{eq:16}
\end{align}
where the inequality holds due to $|\cC \cup \cB|\leq |\cC|+|\cB|$ and the fact that $\cC_{k,1}^l=\cB_{k,1}^l$.

Now, we only need to control the size of $\cB_{k,h}^l$ for each episode $k\in \NN$. For simplicity, we suppose set $\cB_{k,h}^l=\{k_1,..,k_n\}$, where $1\leq k_1\leq k_2\leq ... \leq k_n <k$. According to the definition of level $l_h^k$ in Algorithm \ref{algorithm-mdp} (line \ref{algorithm2-line20} to line \ref{algorithm2-line21}), for  $2\leq i\leq n$, we have
\begin{align}
   \bphi(s_h^{k_i},a_h^{k_i})^{\top}(\bSigma_{k_i,h}^l)^{-1}\bphi(s_h^{k_i},a_h^{k_i})\ge 4^{-l}.\notag
\end{align}
Since $\cB_{k,h}^l\subseteq \cC_{k,h}^l$ holds for all stage $h\in[H]$, all level $l\in \NN$ and all episode $k\in \NN$, we have $\bSigma_{k_i,h}^l\succeq \lambda \Ib+\sum_{j=1}^{i-1} \bphi(s_h^{k_i},a_h^{k_i})\bphi(s_h^{k_i},a_h^{k_i})^\top $ and it implies that 
\begin{align}
     \bphi(s_h^{k_i},a_h^{k_i})^{\top}(\bGamma_{k_i,h}^l)^{-1}\bphi(s_h^{k_i},a_h^{k_i})\ge \bphi(s_h^{k_i},a_h^{k_i})^{\top}(\bSigma_{k_i,h}^l)^{-1}\bphi(s_h^{k_i},a_h^{k_i})\ge4^{-l}.\label{eq:12}
\end{align}
where $\bGamma_{k_i,h}^l=\lambda \Ib+\sum_{j=1}^{i-1} \bphi(s_h^{k_i},a_h^{k_i})\bphi(s_h^{k_i},a_h^{k_i})^\top$ . Thus, taking a summation for \eqref{eq:12} over all $2\leq i\leq n$, we have
\begin{align}
    \sum_{i=1}^n\bphi(s_h^{k_i},a_h^{k_i})^{\top}(\bGamma_{k_i,h}^l)^{-1}\bphi(s_h^{k_i},a_h^{k_i})\ge \sum_{i=2}^n\bphi(s_h^{k_i},a_h^{k_i})^{\top}(\bGamma_{k_i,h}^l)^{-1}\bphi(s_h^{k_i},a_h^{k_i})\ge(n-1)4^{-l},\label{eq:13}
\end{align}
where the first inequality holds due to $\bphi(s_h^{k_1},a_h^{k_1})^{\top}(\bGamma_{k_1,h}^l)^{-1}\bphi(s_h^{k_1},a_h^{k_1})\ge 0$ and the second inequality holds due to \eqref{eq:12}.
On the other hand, according to Lemma \ref{lemma:sum-bonus}, this summation is upper bounded by
\begin{align}
    \sum_{i=1}^n\bphi(s_h^{k_i},a_h^{k_i})^{\top}(\bGamma_{k_i,h}^l)^{-1}\bphi(s_h^{k_i},a_h^{k_i})\leq 2d\log\frac{d\lambda+n}{d \lambda}\leq 2d \log (1+n/d),\label{eq:14}
\end{align}
where the first inequality holds due to Lemma \ref{lemma:sum-bonus} with the fact that $\|\bphi(s,a)\|_2\leq 1$ and the second inequality holds due to $\lambda \ge 1$.
Combining  \eqref{eq:13} and \eqref{eq:14}, we have
\begin{align}
   (n-1)4^{-l} \leq 2d \log (1+n/d),\label{eq:15}
\end{align}
which implies that $|\cB_{k,h}^l|=n\leq 17dl4^{l}$. Finally, substituting \eqref{eq:15} into \eqref{eq:16}, we have.
\begin{align}
    m=|\cC_{k,h}^l|\leq \sum_{j=1}^h |\cB_{k,j}^l|\leq 17dhl4^{l}.
\end{align}
Therefore, we finish the proof of Lemma \ref{lemma:mdp-levelsize}.
\end{proof}
\subsection{Proof of Lemma \ref{lemma:vector-range}}\label{proof:lemma:vector-range}
\begin{proof}[Proof of Lemma \ref{lemma:vector-range}]
In this proof, we only need to show that the norm of vector $\wb_{k,h}^l$ is bounded for each fixed episode $k\in \NN$ and fixed level $l\in \NN$. For simplicity, let $\cC_{k,h}^l=\{k_1,..,k_m\}$ denote the index set $\cC_h^l$ at the beginning of episode $k$, where $1\leq k_1<k_2<..<k_m< k$. According to the definition of weight vector $\wb_{k,h}^l$ in Algorithm \ref{algorithm-mdp} (line \ref{algorithm2-line 7} to line \ref{algorithm2-line9}), we have \begin{align}
  \bSigma_{k,h}^l&=\lambda \Ib+\sum_{i=1}^{m} \bphi(s_h^{k_i},a_h^{k_i})\bphi(s_h^{k_i},a_h^{k_i})^\top,\notag\\
  \bbb_{k,h}^l&=\sum_{i=1}^{m} \bphi(s_h^{k_i},a_h^{k_i})\Big[\reward_h(s_h^{k_i},a_h^{k_i})+\vvalue_{k,h+1}^l(s_{h+1}^{k_i})\Big],\notag\\ \wb_{k,h}^l&=(\bSigma_{k,h}^l)^{-1} \bbb_{k,h}^l.\notag
\end{align}
For simplicity, we omit the subscript $h$ and denote $\reward_{k_i}=\reward_h(s_h^{k_i},a_h^{k_i})+\vvalue_{k,h+1}^l(s_{h+1}^{k_i})$. Then for the norm $\|\wb_{k}^l\|_2$, we have the following inequality
\begin{align}
    \|\wb_{k}^l\|^2_2&=\Big\|\notag (\bSigma_{k}^l)^{-1}\sum_{i=1}^{m}  \bphi(s^{k_i},a^{k_i})\reward_{k_i}\Big\|^2_2\\
    &\leq m\sum_{i=1}^{m} \big\|\notag (\bSigma_{k}^l)^{-1} \bphi(s^{k_i},a^{k_i})\reward_{k_i}\big\|^2_2\notag\\
    &\leq  4mH^2\sum_{i=1}^{m} \big\|\notag (\bSigma_{k}^l)^{-1} \bphi(s^{k_i},a^{k_i})\big\|^2_2\notag\\
    &\leq \frac{4mH^2}{\lambda}\sum_{i=1}^{m}\bphi(s^{k_i},a^{k_i})^{\top}(\bSigma_{k}^l)^{-1} \bphi(s^{k_i},a^{k_i})\notag\\
    &= \frac{4mH^2}{\lambda}\text{tr}\Big((\bSigma_{k}^l)^{-1}\sum_{i=1}^{m}\bphi(s^{k_i},a^{k_i})^{\top} \bphi(s^{k_i},a^{k_i})\Big),\label{eq:616}
\end{align}
where the first inequality holds due to Cauchy-Schwarz inequality, the second inequality holds due to $ |\reward_{k_i}|\leq 2H$ and the last inequality holds due to $\bSigma_{k}^l\succeq \lambda I$. Now, we assume the eigen-decomposition of matrix $\sum_{i=1}^{m}\bphi(s^{k_i},a^{k_i})^{\top} \bphi(s^{k_i},a^{k_i})$ is $Q^{\top}\Lambda Q$ and we have
\begin{align}
    \text{tr}\Big((\bSigma_{k}^l)^{-1}\sum_{i=1}^{m}\bphi(s^{k_i},a^{k_i})^{\top} \bphi(s^{k_i},a^{k_i})\Big)&=\text{tr}\big((Q^{\top}\Lambda Q+\lambda I)^{-1}Q^{\top}\Lambda Q\big)\notag\\
    &=\text{tr}\big((\Lambda+\lambda I)^{-1}\Lambda\big)\notag\\
    &=\sum_{i=1}^d\frac{\Lambda_i}{\Lambda_i+\lambda}\notag\\
    &\leq d.\label{eq:17}
\end{align}
Substituting \eqref{eq:17} into \eqref{eq:616}, we have
\begin{align}
    \|\wb_{k}^l\|^2_2&\leq \frac{4mH^2}{\lambda}\text{tr}\Big((\bSigma_{k}^l)^{-1}\sum_{i=1}^{m}\bphi(s^{k_i},a^{k_i})^{\top} \bphi(s^{k_i},a^{k_i})\Big)\notag\\
    &\leq \frac{4mH^2d}{\lambda}\notag\\
    &\leq \frac{68d^2H^3l4^l}{\lambda},\notag
\end{align}
where the first inequality holds due to \eqref{eq:616}, the second inequality holds due to \eqref{eq:17} and the last inequality holds due to Lemma \ref{lemma:mdp-levelsize}.
Thus, we finish the proof of Lemma \ref{lemma:vector-range}

\end{proof}

\subsection{Proof of Lemma \ref{lemma:mdp-concentration}}\label{proof:lemma:mdp-concentration}

In this section, we provide the proof of Lemma \ref{lemma:mdp-concentration}. For each level $l \in \NN$, we first denote the function class $\mathcal{V}_l$ as
\begin{align}
    \mathcal{V}_l=\Bigg\{V\bigg|V(\cdot)&=\max_{a}\min_{1\leq i\leq l}\min\bigg(H,\wb_{i}^\top\bphi(\cdot,a)+\beta_l\sqrt{\bphi(\cdot,a)^\top \bSigma_{i}^{-1}\bphi(\cdot,a)}\bigg),\notag\\
    &\|\wb_i\|_2\leq 9d2^l \sqrt{H^3l},\bSigma_i \succeq I\Bigg\}.\label{eq:definition-function-set}
\end{align}
Therefore, for all episode $k\in K$ and stage $h\in[H]$, according to Lemma \ref{lemma:vector-range}, we have $\|\wb_{k,h}^l\|\leq 9d2^l \sqrt{H^3l}$ and it implies that the estimated value function $\vvalue_{k,h}^l \in \mathcal{V}_l$. For any function $V \in \mathcal{V}_l$, we have the following concentration property.
\begin{lemma}(Lemma D.4, \citep{jin2019provably})\label{lemma:uni-coverge}
Let $\{x_k\}_{k=1}^{\infty}$ be a real-valued stochastic process on state space $\cS$ with corresponding filtration $ \{\mathcal{F}_k\}_{k=1}^{\infty}$. Let $\{\bphi_k\}_{k=1}^{\infty}$ be an $\RR^d$-valued stochastic process where $\bphi_k \in \mathcal{F}_{k-1}$ and $\|\bphi_k\|_2\leq 1$. For any $k\ge 0$, we define $\bSigma_k= I+\sum_{i=1}^k\bphi_i\bphi_i^{\top}$.
Then with probability at least $1-\delta$, for all $k \in \NN$ and all function $V\in \mathcal{V}$ with $\max_{s}|V(x)|\leq H$, we have
\begin{align}
    \bigg\| \sum_{i=1}^k \bphi_i \Big\{ V(x_i)-\EE\big[V(x_i)|\mathcal{F}_{i-1}\big]\Big\}\bigg\|_{\bSigma_k^{-1}}^2\leq 4H^2\bigg[\frac d2\log(k+1) +\log \frac{\mathcal{N}_\epsilon}{\delta}\bigg]+{8k^2\epsilon^2},\notag
\end{align}
where $\mathcal{N}_\epsilon$ is the $\epsilon$-covering number of the function class $\mathcal{V}$ with respect to the distance function $\text{dist}(V_1,V_2)=\max_{s}|V_1(s)-V_2(s)|$. 
\end{lemma}

Furthermore, for each function class $\mathcal{V}_l$, the covering number $\mathcal{N}_\epsilon$ of $\mathcal{V}_l$ can be upper bounded by following Lemma.
\begin{lemma}\label{lemma:covering-number}
For each function class $\mathcal{V}_l$, we define the distance between two function $V_1$ and $V_2$ as $V_1,V_2\in \mathcal{V}_l$ as $dist(V_1,V_2)=\max_{s}|V_1(s)-V_2(s)|$. With respect to this distance function, the $\epsilon$-covering number $\mathcal{N}_\epsilon$ of the function class $\mathcal{V}_l$ can be upper bounded by
\begin{align}
   \log \mathcal{N}_\epsilon \leq  dl \log (1+36d2^l\sqrt{H^3l}/\epsilon) + d^2l \log (1+8\sqrt{d}\beta_l^2/\epsilon^2).\notag
\end{align}
\end{lemma}
\begin{proof}
See Appendix~\ref{proof:lemma:covering-number}.
\end{proof}

\begin{proof}[Proof of Lemma 
\ref{lemma:mdp-concentration}]
Similar to the proof of Lemma \ref{lemma:uniform-ucb-bandit}, we first focus on a fixed level $l \in \NN$ and a fixed stage $h\in [H]$. Now, we denote $k_0=0$, and for $i\in \NN$, we denote $k_i$ as the minimum index of the episode where the action is added to the set $\cC_h^l$:
\begin{align}
    k_i&=\min \big\{k: k>k_{i-1}, l_h^{k}=l\big\}.\label{eq:mdp-filtration-bandit}
\end{align}
Now, we consider the $\sigma$-algebra filtration $\mathcal{F}_i=\sigma(s_{1}^1,..,s_H^1,s_1^2,..,s_H^2,..,s_1^{k_{i+1}},..,s_{h}^{k_{i+1}})$ that contain all randomness before receive the reward $\reward_h(s_h^{k_{i+1}},a_h^{k_{i+1}})$ and next state $s_{h+1}^{k_{i+1}}$ at episode $k_{i+1}$. By this definition, $\bphi(s_h^{k_{i}},a_h^{k_{i}})$ is $\mathcal{F}_{i-1}$-measurable and the next state $s_{h+1}^{k_{i}}$ is $\mathcal{F}_{i}$-measurable. Since the randomness in this filtration only comes from the stochastic state transition process $s_{h+1}\sim \PP_h(\cdot|s_h,a_h)$ and we determine the level $l_h^k$ before receive the reward $\reward_h(s_h^{k},a_h^{k})$ and next state $s_{h+1}^k$ at episode $k$, for any fixed value function $\vvalue \in \mathcal{V}_l$, we have
\begin{align}
  \EE\big[\vvalue(s_{h+1}^{k_{i}})|\mathcal{F}_{i-1}\big]= [\PP_h\vvalue](s_{h}^{k_{i}},a_{h}^{k_{i}}) .\label{eq:29}
\end{align}
According to Lemma \ref{lemma:uni-coverge} with probability at least $1-\delta/\big(H2l(l+1)\big)$, for all number $i \in \NN$ and all function $\vvalue \in \mathcal{V}_l$, we have
\begin{align}
    &\bigg\|\sum_{j=1}^{i}\bphi(s_h^{k_j},a_h^{k_j})\big[\vvalue(s_{h+1}^{k_j})-[\PP_h\vvalue](s_h^{k_j},a_h^{k_j}) \big]\bigg\|_{(\bSigma_{k_{i+1},h}^l)^{-1}}\notag\\
    &= \bigg\|\sum_{j=1}^{i}\bphi(s_h^{k_j},a_h^{k_j})\big[\vvalue(s_{h+1}^{k_j})- \EE\big[\vvalue(s_{h+1}^{k_{i+1}})|\mathcal{F}_{i-1}\big]\bigg\|_{(\bSigma_{k_{i+1},h}^l)^{-1}}\notag\\\
    &\leq 4H^2\bigg[\frac d2\log(i+1) +\log \frac{\mathcal{N}_\epsilon}{\delta/\big(H2l(l+1)\big)}\bigg]+{8i^2\epsilon^2}\notag\\
    &\leq 4H^2\bigg[\frac d2\log(i+1) + dl \log (1+36d2^l\sqrt{H^3l}/\epsilon) + d^2l \log (1+8\sqrt{d}\beta_l^2/\epsilon^2) \notag\\
    &\qquad+\log(2Hl(l+1)/\delta)\bigg] +{8i^2\epsilon^2},\label{eq:30}
\end{align}
where the first inequality holds due to Lemma \ref{lemma:uni-coverge} and the second inequality holds due to Lemma \ref{lemma:covering-number}. Furthermore, Lemma \ref{lemma:mdp-levelsize} show that the size of set $|\cC_h^l|$ is upper bounded by $|\cC_h^l|\leq 17dlH4^{l}$. This reuslt implies that $k_{17dlH4^{l}+1}=\infty$ and we only need to consider episode $k_i$ for $i\leq 17dlH4^{l}+1$. Now, we choose $\epsilon=1/(17l4^{l})$, then for all episode $k_{i}< k\leq k_{i+1}$ and function $\vvalue=\vvalue_{k,{h+1}}^l \in \mathcal{V}_l$, we have
\begin{align}
    &\bigg\|\sum_{j=1}^{i}\bphi(s_h^{k_j},a_h^{k_j})\big[\vvalue_{k,h+1}^l(s_{h+1}^{k_j})-[\PP_h\vvalue_{k,h+1}^l](s_h^{k_j},a_h^{k_j}) \big]\bigg\|^2_{(\bSigma_{k,h}^l)^{-1}}\notag\\
    &\leq 4H^2\bigg[\frac {dl}2\log(69dlH) +2dl^2 \log(1+2448d\sqrt{H^3l^3}) \notag \\
    &\qquad + d^2l^2 \log (1+36992\sqrt{d}l^2\beta_l^2)+2\log(4lH/\delta)\bigg]+{8d^2H^2},\label{eq:31}
\end{align}
where the first inequality holds due to \eqref{eq:30} with the fact that $\bSigma_{k,h}^l$ does not change for $k_{i}< k\leq k_{i+1}$ and $i\leq 17dlH4^{l}+1$. Finally, taking an union bound for all level $l\in \NN$ and all stage $h\in[H]$, with probability at $1-\delta/2$, for all level $l\in \NN$, all stage $h\in[H]$ and all episode $k\in \NN$, we have
\begin{align}
   \bigg\|\sum_{i\in \cC_{k,h}^l}\bphi(s_h^i,a_h^i)\big[\vvalue_{k,h+1}^l(s_{h+1}^i)-[\PP_h\vvalue_{k,h+1}^l](s_h^i,a_h^i) \big]\bigg\|^2_{(\bSigma_{k,h}^l)^{-1}}\leq C d^2H^2l^2 \log(dlH\beta_l^2/\delta)\notag,
\end{align}
where $C$ is a large absolute constant. Thus, we finish the proof of Lemma \ref{lemma:mdp-concentration}.
\end{proof}

\subsection{Proof of Lemma \ref{lemma:mdp-transition}}\label{proof:lemma:mdp-transition}

\begin{lemma}\label{lemma:qvalue-linear}[Lemma B.1, \citep{jin2019provably}]
Under Assumption \ref{assumption-linear}, for any fixed policy $\pi$, there exists a series of vectors $\{\wb_h^{\pi}\}_{h=1}^H$, such that for all state-action pair $(s,a)\in \cS\times \cA$ and all stage $h\in[H]$, we have
\begin{align}
    \qvalue_h^{\pi}(s,a)=(\wb_h^{\pi})^{\top}\bphi(s,a),\|\wb_h^\pi\|\leq 2H\sqrt{d}.\notag
\end{align}
\end{lemma}
\begin{proof}[Proof of Lemma \ref{lemma:mdp-transition}]
For simplicity, let $\cC_{k,h}^l=\{k_1,..,k_m\}$ denote the index set $\cC_h^l$ at the beginning of episode $k$, where $1\leq k_1<k_2<..<k_m< k$. According to Lemma \ref{lemma:qvalue-linear}, for each fixed policy $\pi$, there exists a vector $\wb_h^{\pi}$ such that
\begin{align}
 (\wb_h^{\pi})^{\top}\bphi(s,a)=\qvalue_h^{\pi}(s,a)=\reward_h(s,a)+\big[\PP_h \vvalue_{h+1}^{\pi}\big](s,a).\label{eq:18}
\end{align}
 According to the definition of vector $\wb_{k,h}$ in Algorithm \ref{algorithm-mdp} (line \ref{algorithm2-line 7} to line \ref{algorithm2-line9}), we have \begin{align}
  \bSigma_{k,h}^l&=\lambda \Ib+\sum_{i=1}^{m} \bphi(s_h^{k_i},a_h^{k_i})\bphi(s_h^{k_i},a_h^{k_i})^\top,\notag\\
  \bbb_{k,h}^l&=\sum_{i=1}^{m} \bphi(s_h^{k_i},a_h^{k_i})\Big[\reward_h(s_h^{k_i},a_h^{k_i})+\vvalue_{k,h+1}^l(s_{h+1}^{k_i})\Big],\notag\\ \wb_{k,h}^l&=(\bSigma_{k,h}^l)^{-1} \bbb_{k,h}^l.\label{eq:19}
\end{align}
For simplicity, we omit the subscript $l$ and combining \eqref{eq:18} and \eqref{eq:19}, we have

\begin{align}
    \wb_{k,h} - \wb_{h}^{\pi}& = \bSigma_{k,h}^{-1}\sum_{i=1}^m\bphi(s_h^{k_i}, a_h^{k_i})\big[\reward_h(s_h^{k_i}, a_h^{k_i}) +   \vvalue_{k,h+1}(s_{h+1}^{k_i})\big]-  \wb_h^{\pi}\notag \\
    &=\bSigma_{k,h}^{-1}\bigg[-\lambda  \wb_h^{\pi}-\sum_{i=1}^m\bphi(s_h^{k_i}, a_h^{k_i})\bphi(s_h^{k_i}, a_h^{k_i})^{\top} \wb_h^{\pi} \notag \\
    &\qquad +\sum_{i=1}^m\bphi(s_h^{k_i}, a_h^{k_i})\big[\reward_h(s_h^{k_i}, a_h^{k_i}) +   \vvalue_{k,h+1}(s_{h+1}^{k_i})\big]\bigg]\notag\\
    &=\bSigma_{k,h}^{-1}\bigg[-\lambda  \wb_h^{\pi} + \sum_{i=1}^m\bphi(s_h^{k_i}, a_h^{k_i})\Big( \vvalue_{k,h+1}(s_{h+1}^{k_i})- [\PP_h\vvalue_{h+1}^{\pi}](s_h^{k_i}, a_h^{k_i})\Big) \bigg]\notag \\
    & = \underbrace{-\lambda \bSigma_{k,h}^{-1}  \wb_h^{\pi}}_{I_1} + \underbrace{ \bSigma_{k,h}^{-1}\sum_{i=1}^m\bphi(s_h^{k_i}, a_h^{k_i})\Big( \vvalue_{k,h+1}(s_{h+1}^{k_i})- [\PP_h\vvalue_{k,h+1}](s_h^{k_i}, a_h^{k_i})\Big)}_{I_2}\notag \\
    &\qquad + \underbrace{ \bSigma_{k,h}^{-1}\sum_{i=1}^m\bphi(s_h^{k_i}, a_h^{k_i})\big[\PP_h(\vvalue_{k,h+1} - \vvalue_{h+1}^{\pi})\big](s_h^{k_i}, a_h^{k_i})}_{I_3}\notag,
\end{align}
where the third equality holds due to \eqref{eq:18}.
For the term $I_1$ and any state-action pair $(s,a) \in \cS\times \cA$, we have
\begin{align}
    \Big|\big\la I_1, \bphi(s,a) \big\ra\Big| &=\big|\lambda \bphi(s,a)^{\top} \bSigma_{k,h}^{-1}  \wb_h^{\pi}\big| \notag\\
    &\leq \lambda\big\|\bphi(s,a)^{\top} \bSigma_{k,h}^{-1}\big\|_2\| \wb_h^{\pi}\|_2\notag\\ 
    &\leq \sqrt{\lambda}\| \wb_h^{\pi}\|_2\sqrt{\bphi(s,a)^\top\bSigma_{k,h}^{-1} \bphi(s,a)}\notag\\
    &\leq 2{H\sqrt{d\lambda}}\sqrt{\bphi(s,a)^\top\bSigma_{k,h}^{-1} \bphi(s,a)},\label{eq:i1}
\end{align}
where the first inequality holds due to Cauchy-Schwarz inequality, the second inequality holds due to $\bSigma_{k,h}\succeq \lambda I$ and the third inequality holds due to Lemma \ref{lemma:qvalue-linear}. For the term $I_2$ and any state-action pair $(s,a) \in \cS\times \cA$, according to Lemma \ref{lemma:mdp-concentration}, we have
\begin{align}
    \big|\la I_2, \bphi(s,a) \ra \big|&\leq  \sqrt{\bphi(s,a)^\top\bSigma_{k,h}^{-1} \bphi(s,a)} \notag\\
    &\qquad \cdot
    \bigg\|\sum_{i=1}^m\bphi(s_h^{k_i}, a_h^{k_i})\big[\vvalue_{k,h+1}(s_{h+1}^{k_i})-[\PP_h\vvalue_{k,h+1}](s_h^{k_i}, a_h^{k_i}) \big]\bigg\|_{\bSigma_{k,h}^{-1}}\notag\\
    &\leq C dHl\sqrt{ \log(dlH\beta_l^2/\delta)}\sqrt{\bphi(s,a)^\top\bSigma_{k,h}^{-1} \bphi(s,a)},\label{eq:i2}
\end{align}
where the first inequality holds due to Cauchy-Schwarz inequality and the second inequality holds due to Lemma \ref{lemma:mdp-concentration}. For the term $I_3$ and any state-action pair $(s,a) \in \cS\times \cA$, we have
\begin{align}
    \la \bphi(s,a), I_3\ra &=  \bigg\la \bphi(s,a), \bSigma_{k,h}^{-1}\sum_{i=1}^m\bphi(s_h^{k_i}, a_h^{k_i})\big[\PP_h(\vvalue_{k,h+1} - \vvalue_{h+1}^{\pi})\big](s_h^{k_i}, a_h^{k_i})\bigg\ra\notag \\
    & =  \bigg\la \bphi(s,a), \bSigma_{k,h}^{-1}\sum_{i=1}^m\bphi(s_h^{k_i}, a_h^{k_i})\bphi(s_h^{k_i}, a_h^{k_i})^\top \int (\vvalue_{k,h+1} - \vvalue_{h+1}^{\pi}) (s') d\bmu_h(s')\bigg\ra\notag\\ 
    & = \underbrace{ \bigg\la \bphi(s,a), \int (\vvalue_{k,h+1} - \vvalue_{h+1}^{\pi}) (s') d\bmu_h(s')\bigg \ra}_{J_1}\notag\\
    &\qquad-   \underbrace{\lambda    \bigg\la \bphi(s,a), \bSigma_{k,h}^{-1}\int (\vvalue_{k,h+1}-\vvalue_{h+1}^{\pi}) (s') d\bmu_h(s')\bigg \ra}_{J_2},\notag
\end{align}
For term $J_1$ , we have
\begin{align}
    J_1&=  \bigg\la \bphi(s,a), \int (\vvalue_{k,h+1} - \vvalue_{h+1}^{\pi}) (s') d\bmu_h(s')\bigg \ra\notag\\
    &=  \int\bigg\la \bphi(s,a), (\vvalue_{k,h+1} - \vvalue_{h+1}^{\pi}) (s') \bigg \ra d\bmu_h(s')\notag\\
    &=  \int \PP_h(s'|s,a)(\vvalue_{k,h+1} - \vvalue_{h+1}^{\pi}) (s') \ra d s'\notag\\\
    &= \big[\PP_h(\vvalue_{k,h+1} - \vvalue_{h+1}^{\pi})\big](s,a).\label{eq:j1}
\end{align} 
For term $J_2$, we have
\begin{align}
    \big|J_2\big|&=\lambda   \Bigg| \bigg\la \bphi(s,a), \bSigma_{k,h}^{-1}\int (\vvalue_{k,h+1} - \vvalue_{h+1}^{\pi}) (s') d\bmu_h(s')\bigg \ra\Bigg|\notag\\
    &\leq \lambda   \big\|\bphi(s,a)^{\top}\bSigma_{k,h}^{-1}\big\|_2 \bigg\|\int (\vvalue_{k,h+1} - \vvalue_{h+1}^{\pi}) (s') d\bmu_h(s')\bigg\|_2\notag\\
    &\leq \sqrt{d}\lambda  \big\|\bphi(s,a)^{\top}\bSigma_{k,h}^{-1}\big\|_2\max_{s'}\big|(\vvalue_{k,h+1} - \vvalue_{h+1}^{\pi}) (s')\big|\notag\\
    &\leq 2H\sqrt{d}\lambda \big\|\bphi(s,a)^{\top}\bSigma_{k,h}^{-1}\big\|_2\notag\\
    &\leq 2H\sqrt{d\lambda}  \sqrt{\bphi(s,a)^\top\bSigma_{k,h}^{-1} \bphi(s,a)},\label{eq:j2}
\end{align}
where the first inequality holds due to Cauchy-Schwarz inequality, the second inequality holds due to Assumption \ref{assumption-linear}, the third inequality holds because of  $\big|(\vvalue_{k,h+1} - \vvalue_{h+1}^{\pi}) (s')\big|\leq 2H$ and the last inequality holds due to $\bSigma_{k,h}\succeq \lambda I$.
Combining  \eqref{eq:i1},\eqref{eq:i2},\eqref{eq:j1},\eqref{eq:j2} with the fact that $\lambda=1$, we have
\begin{align}
    &\Big|\la \bphi(s,a), \wb_{k,h} \ra - \qvalue_h^{\pi}(s,a) -  \big[\PP_h(\vvalue_{k,h+1} - \vvalue_{h+1}^{\pi})\big](s,a)\Big|\notag\\
    &= |J_2+\la I_1,\bphi(s,a)\ra+\la I_2,\bphi(s,a)\ra|\notag\\
    &\leq \Big(C dHl\sqrt{ \log(dlH\beta_l^2/\delta)}+4H\sqrt{d}\Big) \sqrt{\bphi(s,a)^{\top}(\bSigma_{k,h})^{-1}\bphi(s,a)}.\notag
\end{align}
Notice that there exists a large constant $C'$ such that for all level $l\in \NN$ with parameter $\beta_l=C' dHl \sqrt{ \log(dlH/\delta)}$, the following inequality hods:
\begin{align}
    C dHl\sqrt{ \log(dlH\beta_l^2/\delta)}+4H\sqrt{d} \leq  C' dHl \sqrt{ \log(dlH/\delta)}.\label{eq:100}
\end{align}
When \eqref{eq:100} holds, we further have
\begin{align}
    &\big|\la \bphi(s,a), \wb_{k,t} \ra - \qvalue_h^{\pi}(s,a) -  \big[\PP_h(\vvalue_{k,h+1} - \vvalue_{h+1}^{\pi})\big](s,a)\big|\notag\\
    &\leq  \Big(C dHl\sqrt{ \log(dlH\beta_l^2/\delta)}+3H\sqrt{d}\Big) \sqrt{\bphi(s,a)^{\top}(\bSigma_{k,h})^{-1}\bphi(s,a)}\notag\\
    &\leq C' dHl \sqrt{ \log(dlH/\delta)}\sqrt{\bphi(s,a)^{\top}(\bSigma_{k,h})^{-1}\bphi(s,a)}\notag\\
    &= \beta_l\sqrt{\bphi(s,a)^{\top}(\bSigma_{k,h})^{-1}\bphi(s,a)}.\notag
\end{align}
Thus, we finish the proof of Lemma \ref{lemma:mdp-transition}.
\end{proof}

\subsection{Proof of Lemma \ref{lemma:mdp-upper-confidence}}\label{proof:lemma:mdp-upper-confidence}

\begin{proof}[Proof of Lemma \ref{lemma:mdp-upper-confidence}]
Now, we use induction to prove this lemma. First, we prove the base case. For all state $s\in \cS$ and level $l\in \NN$, we have $\vvalue_{k,H+1}^l(s)=0=\vvalue_{H+1}^*(s)$. Second, if $\vvalue_{k,h+1}^l(s)\ge \vvalue_{h+1}^*(s)$ holds for all state $s\in \cS$ and level $l\in \NN$ at stage $h+1$, then for any state $s\in \cS$ and level $l\in \NN$  at stage $h$, we have
\begin{align}
    (\wb_{k,h}^l)^{\top}\bphi(s,a)+\beta_l \sqrt{\bphi(s,a)^{\top}(\Sigma_{k,h}^l)^{-1}\bphi(s,a)}- \qvalue_{h}^{*}(s,a) \ge \big[\PP_h(\vvalue_{k,h+1}^l - \vvalue_{h+1}^{*})\big](s,a)\ge 0,\notag
\end{align}
where the first inequality holds due to Lemma \ref{lemma:mdp-transition} and the second inequality holds due to the induction assumption. Furthermore, the optimal value function is upper bounded by $\qvalue_{h}^{*}(s,a)\leq H$ and it implies that
\begin{align}
   \qvalue_{h}^{*}(s,a)\leq \min\Big((\wb_{k,h}^l)^{\top}\bphi(s,a)+\beta_l \sqrt{\bphi(s,a)^{\top}(\Sigma_{k,h}^l)^{-1}\bphi(s,a)},H \Big)=\qvalue_{k,h}^{l}(s,a). \label{eq:20}
\end{align}
Thus, for each level $l\in \NN$ and state $s\in \cS$, we have
\begin{align}
\vvalue_{k,h}^{l}(s)&=\max_{a} \min_{1\leq i\leq l} \qvalue_{k,h}^{i}(s,a)\notag\\
&\ge \max_{a} \min_{1\leq i\leq l} \qvalue_{h}^{*}(s,a)\notag\\
&=\max_{a} \qvalue_{h}^{*}(s,a)\notag\\
&=\vvalue_h^*(s),\notag
\end{align}
where the inequality holds due to \eqref{eq:20}. Finally, by induction, we finish the proof of Lemma \ref{lemma:mdp-upper-confidence}.
\end{proof}

\section{Auxiliary Lemmas}
\begin{lemma}[Azuma–Hoeffding inequality, \citep{cesa2006prediction}]\label{lemma:azuma}
Let $\{x_i\}_{i=1}^n$ be a martingale difference sequence with respect to a filtration $\{\cG_{i}\}$ satisfying $|x_i| \leq M$ for some constant $M$, $x_i$ is $\cG_{i+1}$-measurable, $\EE[x_i|\cG_i] = 0$. Then for any $0<\delta<1$, with probability at least $1-\delta$, we have 
\begin{align}
    \sum_{i=1}^n x_i\leq M\sqrt{2n \log (1/\delta)}.\notag
\end{align} 
\end{lemma}

\subsection{Proof of Lemma \ref{lemma:covering-number}}\label{proof:lemma:covering-number}
We need the following Lemma:
\begin{lemma}[Lemma D.5, \citep{jin2019provably}]\label{lemma:ball-cover}
For an Euclidean ball with radius $R$ in $\RR^d$, the  $\epsilon$-covering number of this ball is upper bounded by $(1+2R/\epsilon)^d$.
\end{lemma}
\begin{proof}[Proof of Lemma \ref{lemma:covering-number}]
For any two function $\vvalue_1, \vvalue_2 \in \mathcal{V}_l$, according to the definition of function class $\mathcal{V}_l$, we have
\begin{align}
  \vvalue_1(\cdot)=\max_{a}\min_{1\leq i\leq l}\min\bigg(H,\wb_{1,i}^\top\bphi(\cdot,a)+\beta_l\sqrt{\bphi(\cdot,a)^\top \bGamma_{1,i}\bphi(\cdot,a)}\bigg) ,\notag\\
    \vvalue_2(\cdot)=\max_{a}\min_{1\leq i\leq l}\min\bigg(H,\wb_{2,i}^\top\bphi(\cdot,a)+\beta_l\sqrt{\bphi(\cdot,a)^\top \bGamma_{2,i}\bphi(\cdot,a)}\bigg),\notag 
\end{align}
where $\|\wb_{1,i}\|_2,\|\wb_{2,i}\|_2\leq 9d2^l \sqrt{H^3l}$ and $\bGamma_{1,i}, \bGamma_{2,i} \preceq  \Ib$. Since all of the functions $\max_a$, $\min_{1\leq i\leq l}$ and $\min(H,\cdot)$ are contraction functions, we have
\begin{align}
    \text{dist}(\vvalue_1,\vvalue_2)&=\max_{s\in \cS} \big|\vvalue_1(s)-\vvalue_2(s)\big|\notag\\
    &\leq \max_{1\leq i\leq l, s\in \cS, a\in \cA}  
    \Big|  \wb_{1,i}^\top\bphi(s,a)+\beta_l\sqrt{\bphi(s,a)^\top \bGamma_{1,i}\bphi(s,a)}\notag\\
    &\qquad - \wb_{2,i}^\top\bphi(s,a)-\beta_l\sqrt{\bphi(s,a)^\top \bGamma_{2,i}\bphi(s,a)}\Big|\notag\\
    &\leq \beta_l \max_{1\leq i\leq l, s\in \cS, a\in \cA} \Big|\sqrt{\bphi(s,a)^\top \bGamma_{1,i}\bphi(s,a)}- \sqrt{\bphi(s,a)^\top \bGamma_{2,i}\bphi(s,a)} \Big|\notag\\
    &\qquad + \max_{1\leq i\leq l, s\in \cS, a\in \cA} \big |(\wb_{1,i}-\wb_{2,i})^{\top}\bphi(s,a)\big|\notag\\
    &\leq \beta_l \max_{1\leq i\leq l, s\in \cS, a\in \cA} \Big|\sqrt{\bphi(s,a)^\top (\bGamma_{1,i}-\bGamma_{2,i})\bphi(s,a)}\Big|\notag\\
    &\qquad +\max_{1\leq i\leq l, s\in \cS, a\in \cA} \big |(\wb_{1,i}-\wb_{2,i})^{\top}\bphi(s,a)\big|\notag \notag\\
    &\leq \beta_l \max_{1\leq i\leq l}\sqrt{\|\bGamma_{1,i}-\bGamma_{2,i}\|_{F}}+ \max_{1\leq i\leq l}\|\wb_{1,i}-\wb_{2,i}\|_2,\label{eq:27}
\end{align}
where the first inequality holds due to the contraction property, the second inequality holds due to the fact that $\max_{x}|f(x)+g(x)|\leq \max_{x}|f(x)|+\max_{x}|g(x)|$, the third inequality holds due to $|\sqrt{x}-\sqrt{y}|\ge |\sqrt{x}-\sqrt{y}|$ and the last inequality holds due to the fact that $\|\bphi(s,a)\|_2\leq 1$.
Now, we denote $\mathcal{C}_{\wb}$ as a $\epsilon/2$-cover of the set $\big\{\wb \in \RR^d\big| \|\wb\|_2\leq 9d2^l\sqrt{H^3l}\big\}$ and $\mathcal{C}_{\bGamma}$ as a $\epsilon^2/(4\beta_l^2)$-cover of the set $\{\bGamma\in \RR^{d\times d}\big|\|\bGamma\|_{F}\leq \sqrt{d} \}$ with respect to the Frobenius norms. Thus, according to Lemma \ref{lemma:ball-cover}, we have following property:
\begin{align}
    |\mathcal{C}_{\wb}|\leq \big(1+36d2^l\sqrt{H^3l}/\epsilon\big)^d, |\mathcal{C}_{\bGamma}|\leq \big(1+8\sqrt{d}\beta_l^2/\epsilon^2\big)^{d^2}.\label{eq:28}
\end{align}
By the definition of covering number, for any function $V_1\in \mathcal{V}_l$ with parameters $\wb_{1,i}, \bGamma_{1,i}(1\leq i\leq l)$, there exists other parameters $\wb_{2,i}, \bGamma_{2,i}(1\leq i\leq l)$ such that $\wb_{2,i} \in \mathcal{C}_{\wb}, \bGamma_{2,i} \in \mathcal{C}_{\bGamma}$ and $\|\wb_{2,i}-\wb_{1,i}\|_2\leq \epsilon/2, \|\bGamma_{2,i}-\bGamma_{1,i}\|_{F}\leq \epsilon^2/(4\beta_l^2)$. Thus, we have
\begin{align}
     \text{dist}(\vvalue_1,\vvalue_2)\leq \beta_l \max_{1\leq i\leq l}\sqrt{\|\bGamma_{1,i}-\bGamma_{2,i}\|_{F}}+ \max_{1\leq i\leq l}\|\wb_{1,i}-\wb_{2,i}\|_2 \leq \epsilon,\notag
\end{align}
where the inequality holds due to \eqref{eq:27}.
Therefore, the $\epsilon$-covering number of function class $\mathcal{V}_l$ is bounded by $\mathcal{N}_\epsilon\leq |\mathcal{C}_{\wb}|^l\cdot |\mathcal{C}_{\bGamma}|^l$ and it implies
\begin{align}
    \log \mathcal{N}_\epsilon &\leq dl \log (1+36d2^l\sqrt{H^3l}/\epsilon) + d^2l \log (1+8\sqrt{d}\beta_l^2/\epsilon^2),\notag
\end{align}
where the first inequality holds due to \eqref{eq:28}. Thus, we finish the proof of Lemma \ref{lemma:covering-number}.
\end{proof}

\end{document}